\documentclass{article}
\usepackage[english]{babel}
\usepackage{arxiv/Macros}
\usepackage{mathrsfs}
\usepackage{wrapfig}
\usepackage{subcaption}

\title{SGD at the Edge of Stability: The Stochastic Sharpness Gap}
\author{
    Fangshuo Liao \\ Rice CS \\ Fangshuo.Liao@rice.edu
    \and
    Afroditi Kolomvaki \\ Rice CS \\ ak203@rice.edu
    \and
    Anastasios Kyrillidis \\ Rice CS, Ken Kennedy Institute \\ anastasios@rice.edu
}

\date{}

\begin{document}
\maketitle

\begin{abstract}
    When training neural networks with full-batch gradient descent (\gd) and step size~$\eta$, the largest eigenvalue of the Hessian---the sharpness $ S(\param)$---rises to $2/\eta$ and hovers there, a phenomenon termed the \emph{Edge of Stability} (\eos). \citet{damian2023selfstab} showed that this behavior is explained by a \emph{self-stabilization} mechanism driven by third-order structure of the loss, and that \gd{} implicitly follows projected gradient descent (\pgd) on the constraint $ S(\param)\leq 2/\eta$. For mini-batch stochastic gradient descent (\sgd), the sharpness stabilizes \emph{below}~$2/\eta$, with the gap widening as the batch size decreases; yet no theoretical explanation exists for this suppression.
  
    We introduce \emph{stochastic self-stabilization}, extending the self-stabilization framework to \sgd. Our key insight is that gradient noise injects variance into the oscillatory dynamics along the top Hessian eigenvector, strengthening the cubic sharpness-reducing force and shifting the equilibrium below $2/\eta$. Following the approach of \citet{damian2023selfstab}, we define \emph{stochastic predicted dynamics} relative to a moving projected gradient descent trajectory and prove a stochastic coupling theorem that bounds the deviation of \sgd{} from these predictions. We derive a closed-form equilibrium sharpness gap: $\Delta  S = \eta \beta \sigma_{ \v{u}}^{2}/(4\alpha)$, where $\alpha$ is the progressive sharpening rate, $\beta$ is the self-stabilization strength, and $\sigma_{ \v{u}}^{2}$ is the gradient noise variance projected onto the top eigenvector. This formula predicts that smaller batch sizes yield flatter solutions, recovers \gd{} when the batch equals the full dataset, and identifies \emph{Batch Sharpness} as the quantity that genuinely saturates at $2/\eta$.
\end{abstract}

\section{Introduction}\label{sec:intro}

\textbf{Motivation.}
The choice of optimizer and its hyperparameters profoundly affects both the training dynamics and the generalization performance of neural networks.
A line of empirical work~\citep{cohen2021eos,jastrzebski2019on,jastrzebski2020breakeven,xing2018walk} has established that gradient-based optimizers routinely operate in regimes that classical optimization theory deems unstable.
The classical \emph{descent lemma} guarantees monotonic loss decrease only when the step size satisfies $\eta < 2/ S(\param)$, where $ S(\param) = \lmax(\nabla^{2} L(\param))$ is the largest eigenvalue of the Hessian of the loss function $L(\cdot)$.
Yet in practice, training proceeds---and succeeds---even when this condition is violated.

For full-batch gradient descent (\gd), \citet{cohen2021eos} documented two striking phenomena.
First, \emph{progressive sharpening}: the sharpness steadily increases throughout early training until it reaches the instability threshold~$2/\eta$.
Second, the \emph{Edge of Stability} (\eos): once the sharpness reaches $2/\eta$, it ceases to grow and instead oscillates in a narrow band around this value, while the loss continues to decrease non-monotonically.

\citet{damian2023selfstab} provided a theoretical explanation for \eos{} by identifying a \emph{self-stabilization} mechanism.
Through a cubic Taylor expansion of the loss, they showed that as the iterates diverge along the top Hessian eigenvector due to instability, the third-order term acts as a restoring force that drives the sharpness back below $2/\eta$.
Their analysis does not merely describe the mechanism heuristically; they define a \emph{constrained trajectory} $\{\paramdagger_t\}$ via projected gradient descent (\pgd) on the set $\calM = \{\param :  S(\param)\leq 2/\eta, \ip{\nabla L(\param)}{ \v{u}(\param)}=0\}$, introduce \emph{predicted dynamics} $\hat{\v{v}}_t$ that capture the oscillatory deviation of \gd{} from $\paramdagger_t$, and prove a \emph{coupling theorem} showing that the true \gd{} trajectory tracks $\paramdagger_t + \hat{\v{v}}_t$ up to small errors.

\medskip \noindent \textbf{Problem Statement.}
For mini-batch \sgd, the situation is both empirically richer and theoretically more opaque.
\citet{cohen2021eos} observed that during \sgd{} training, the sharpness stabilizes \emph{below} the $2/\eta$ threshold, with the plateau level depending strongly on the batch size: smaller batches lead to lower sharpness.
More recently, \citet{andreyev2025eoss} identified \emph{Batch Sharpness}---the expected directional curvature of the mini-batch Hessian along the mini-batch gradient---as the quantity that actually saturates at $2/\eta$ for \sgd, terming this the \emph{Edge of Stochastic Stability} (\eoss).
Meanwhile, \citet{lee2023ias} introduced \emph{Interaction-Aware Sharpness} and showed that it characterizes the onset of oscillatory behavior in \sgd. Despite this progress, a fundamental question remains unanswered:
\begin{quote}
\emph{By what mechanism does \sgd{} self-stabilize, and why does the full-batch sharpness settle at a level below $2/\eta$ that depends on the batch size?}
\end{quote}
This question was explicitly posed as an open problem by both \citet{andreyev2025eoss} and \citet{lee2023ias}.

\medskip \noindent \textbf{Main Contributions.}
We answer this question by extending the self-stabilization framework of \citet{damian2023selfstab} from \gd{} to \sgd.
Our contributions are as follows.

\begin{enumerate}[leftmargin=2em,itemsep=4pt]
\item \textbf{Stochastic Self-Stabilization Mechanism (Section~\ref{sec:mechanism}).}
We show that gradient noise injects additional variance into the oscillatory dynamics along the top Hessian eigenvector.
Through the same cubic nonlinearity that drives \gd{} self-stabilization, this variance acts as an additional sharpness-suppressing force, shifting the equilibrium sharpness below $2/\eta$.

\item \textbf{Stochastic Predicted Dynamics and Coupling (Section~\ref{sec:main_results}).}
Following the methodology of \citet{damian2023selfstab}, we define \emph{enhanced stochastic predicted dynamics} that extend their deterministic predicted dynamics to include gradient noise, and we prove a \emph{stochastic coupling theorem} showing that the true \sgd{} trajectory is well-approximated by the constrained trajectory $\paramdagger_t$ plus the stochastic predicted dynamics.
All Taylor expansions are performed around the moving reference $\paramdagger_t$, ensuring that subleading terms remain controlled.

\item \textbf{Closed-Form Sharpness Gap (Section~\ref{sec:sharpness_gap}).}
From the stationarity conditions of the stochastic predicted dynamics, we derive a prediction for the equilibrium full-batch sharpness under \sgd:
\begin{equation}\label{eq:main_result_intro}
   S_{\mathrm{eq}}
   \approx
  \frac{2}{\eta}
   -
  \underbrace{\frac{\eta \beta \sigma_{ \v{u}}^{2}}{4 \alpha}}_{\displaystyle\Delta S},
\end{equation}
where $\alpha$ is the progressive sharpening coefficient, $\beta = \norm{\nabla S^{\perp}}^{2}$ is the self-stabilization strength, and $\sigma_{ \v{u}}^{2} =  \v{u}^{\top}\v{\Sigma}_{b}  \v{u}$ is the gradient noise variance projected onto the top eigenvector.

\item \textbf{Experimental Verification (Section~\ref{sec:experiments}).}
We verify all predictions on MLPs, CNNs, and ResNets trained on CIFAR-10 with varying batch sizes and step sizes, demonstrating quantitative agreement between the predicted and measured sharpness gaps.
\end{enumerate}

% ════════════════════════════════════════════════════════════
\section{Preliminaries}\label{sec:prelim}
% ════════════════════════════════════════════════════════════

We begin by establishing notation, reviewing the \gd{} self-stabilization theory of~\citet{damian2023selfstab}, and recalling the definition of Batch Sharpness.

\medskip \noindent \textbf{Setup.}
We consider the empirical loss $ L(\param) = \frac{1}{n}\sum_{i=1}^{n}\ell(\v{x}_{i};\param)$ over a training set $\{\v{x}_{i}\}_{i=1}^{n}$, where $\param\in\R^{d}$.
Stochastic gradient descent with constant step size $\eta>0$ and mini-batch size $b$ performs:
\begin{equation}\label{eq:sgd_update}
  \param_{t+1} = \param_t - \eta \v{g}_{B_t}(\param_t),
  \qquad
  \v{g}_{B_t}(\param) = \frac{1}{b}\sum_{i\in B_t}\nabla\ell(\v{x}_i;\param),
\end{equation}
where $B_t\subset\{1,\ldots,n\}$ with $|B_t|=b$ is drawn uniformly at random.
We write $\v{g}_{B_t}(\param_t) = \nabla L(\param_t)+\v{\xi}_t$ where $\v{\xi}_t$ is the gradient noise satisfying $\E{\v{\xi}_t|\param_t}=\v{0}$.
Full-batch \gd{} is the special case $b=n$, where $\v{\xi}_t=\v{0}$ identically.

\medskip \noindent \textbf{Sharpness and its gradient.}
When the top eigenvalue of $ \mat{H}(\param) = \nabla^{2} L(\param)$ is simple\footnote{An eigenvalue \(\lambda\) of a matrix \(\mat{A}\) is called \textit{simple} if it has \textit{algebraic multiplicity equal to 1}.}, the sharpness $ S(\param) = \lmax( \mat{H}(\param))$ is differentiable, and its gradient is given by a fundamental identity relating the third derivative of the loss to the sharpness gradient.

\begin{lemma}[Sharpness Gradient Identity~\citep{damian2023selfstab}]\label{lem:sharpness_grad}
If the top eigenvalue of $\nabla^{2} L(\param)$ is simple with eigenvector $ \v{u}(\param)$, then:
\begin{equation}\label{eq:sharpness_grad}
  \nabla S(\param) = \nabla^{3} L(\param)\bigl( \v{u}(\param), \v{u}(\param)\bigr).
\end{equation}
\end{lemma}

\medskip \noindent \textbf{Progressive sharpening.}
Following \citet{damian2023selfstab}, we define the \emph{progressive sharpening coefficient} as:
\begin{equation}\label{eq:alpha_def}
  \alpha(\param) := -\langle \nabla L(\param), \nabla S(\param)\rangle .
\end{equation}
When $\alpha(\param)>0$, the negative gradient direction $-\nabla L(\param)$ has a positive component along $\nabla S(\param)$, meaning that gradient steps tend to increase the sharpness.
This condition is empirically observed throughout most of training for standard neural networks~\citep{cohen2021eos,damian2023selfstab}.

\medskip \noindent \textbf{The constrained trajectory and self-stabilization under GD.}
We now describe the central construction of~\citet{damian2023selfstab} that serves as the foundation for our extension.
Define the \emph{stable set}:
\begin{equation}\label{eq:calM}
  \calM := \bigl\{\param :  S(\param)\leq 2/\eta \text{ and } \ip{\nabla L(\param)}{ \v{u}(\param)}=0\bigr\}.
\end{equation}
The \emph{constrained trajectory} $\{\paramdagger_t\}$ is defined by projected gradient descent on $\calM$:
\begin{equation}\label{eq:pgd_trajectory}
  \paramdagger_0 := \proj_\calM(\param_0),
  \qquad
  \paramdagger_{t+1} := \proj_\calM\bigl(\paramdagger_t - \eta\nabla L(\paramdagger_t)\bigr),
\end{equation}
where $\proj_\calM(\param):=\arg\min_{\param'\in\calM}\norm{\param-\param'}$ is the orthogonal projection.
\citet[Lemma~I.3]{damian2023selfstab} showed that $ S(\paramdagger_t)=2/\eta$ for all $t$, and the PGD update takes the approximate form $\paramdagger_{t+1}=\paramdagger_t - \eta P_{ \v{u}_t,\nabla S_t}^{\perp}\nabla L_t + O(\epsilon_t^{2}\eta\norm{\nabla L_t})$, where $P_{ \v{u}_t,\nabla S_t}^{\perp}$ projects out the $ \v{u}_t$ and $\nabla S_t$ directions.
That is, the constrained trajectory performs gradient descent with the sharpness-increasing components projected out.

\begin{lemma}[Properties of $\paramdagger_t$, {\citet[Lemma~5, 6]{damian2023selfstab}}]\label{lem:dagger_step}
Under Assumptions~\ref{ass:prog_sharp}--\ref{ass:non_worst} below, for all $t\leq\mathscr{T}$:
\begin{equation}\label{eq:dagger_approx}
   S(\paramdagger_t) = 2/\eta,
  \qquad
  \paramdagger_{t+1} = \paramdagger_t - \eta P_{ \v{u}_t,\nabla S_t}^{\perp}\nabla L_t + O(\epsilon_t^{2}\cdot\eta\norm{\nabla L_t}).
\end{equation}
Moreover, $ L(\paramdagger_{t+1})\leq L(\paramdagger_t)-\Omega(\eta\norm{P_{ \v{u}_t,\nabla S_t}^{\perp}\nabla L_t}^{2})$ (descent lemma for $\paramdagger_t$).
\end{lemma}

The key idea is that progressive sharpening ($\alpha_t>0$) forces the PGD constraint to be active, so $ S(\paramdagger_t)=2/\eta$ exactly and the projection removes the $ \v{u}_t$ and $\nabla S_t$ components from the gradient step. Based on the reference trajectory obtained from the projected gradient descent, we can define the following landscape quantities that will be used in our analysis.

The \gd{} trajectory $\{\param_t\}$ oscillates around $\{\paramdagger_t\}$.
The displacement $\v{v}_t := \param_t - \paramdagger_t$ is captured by the \emph{predicted dynamics} $\hat{\v{v}}_t$, which track the 2D oscillatory system in the $( \v{u}_t,\nabla S_t^{\perp})$ directions (see Section~\ref{sec:mechanism} for a heuristic preview).
The main result of~\citet{damian2023selfstab} is a \emph{coupling theorem}: for $T=O(\epsilon^{-1})$ steps, $\norm{\v{v}_t-\hat{\v{v}}_t}=O(\epsilon\delta)$, and the loss, sharpness, and deviation from $\paramdagger_t$ are all predicted by $\hat{\v{v}}_t$.

\medskip \noindent \textbf{Background: self-stabilization under \gd.} We briefly summarize the self-stabilization mechanism of \citet{damian2023selfstab} for full-batch \gd, as our extension to \sgd{} follows the same conceptual structure.
Consider \gd{} with step size $\eta$: $\param_{t+1} = \param_{t} - \eta\nabla L(\param_{t})$.
Suppose the iterate reaches a point $\paramstar$ where $ S(\paramstar)=2/\eta$.
Define the displacements:
\[
  x_{t} := \langle  \v{u}, \param_{t}-\paramstar\rangle ,
  \qquad
  y_{t} := \langle \nabla S, \param_{t}-\paramstar\rangle ,
\]
where $ \v{u}= \v{u}(\paramstar)$ and $\nabla S = \nabla S(\paramstar)$ are evaluated at $\paramstar$.
Note that $y_{t}$ approximates the sharpness deviation: $ S(\param_{t}) \approx 2/\eta + y_{t}$; this decomposition arises from the specific dynamics at the edge of stability. When gradient descent uses step size \(\eta\) and the sharpness \( S(\param_t)\) approaches \(2/\eta\), the iterates begin to oscillate rather than converge monotonically, and empirically the sharpness hovers around this threshold value. \citet{damian2023selfstab} models this by writing \( S(\param_t) = 2/\eta + y_t\) where \(2/\eta\) is the baseline ``target'' sharpness that characterizes the edge-of-stability regime, and \(y_t\) captures the time-varying deviations that drive the oscillatory behavior and implicit bias toward flatter minima.

% ════════════════════════════════════════════════════════════
\section{The Four Stages of EoS under SGD}\label{sec:mechanism}
% ════════════════════════════════════════════════════════════
In this section, we give a heuristic derivation of the four-stage dynamic of the sharpness under SGD. The four-stage description generalizes the setting in \citet{damian2023selfstab}, and reduced to their setting by letting the noise to be zero.
For the simplicity of the demonstration, the below derivation uses a fixed reference $\paramstar$ and ignores the time dependence of $\alpha,\beta,\delta$.
The rigorous analysis in Section~\ref{sec:main_results} addresses both issues.

\medskip \noindent \textit{Setup and Reference Point.}
Suppose the \sgd{} trajectory first reaches a point where the sharpness is approximately $2/\eta$.
Let $\paramstar$ denote the projection of this point onto the manifold $\calM := \{\param :  S(\param) \leq 2/\eta, \langle \nabla L(\param),  \v{u}(\param)\rangle =0\}$.
We fix the reference point $\paramstar$ and define the displacement vector:
\begin{equation}\label{eq:displacement}
  \v{v}_{t} := \param_{t} - \paramstar,
\end{equation}
with projections $x_{t} := \langle  \v{u}, \v{v}_{t}\rangle $ and $y_{t} := \langle \nabla S, \v{v}_{t}\rangle$, where $ \v{u} =  \v{u}(\paramstar)$ and $\nabla S = \nabla S(\paramstar)$.

\medskip \noindent \textit{Stage 1: Progressive Sharpening under SGD.}
When $\v{v}_{t}$ is small, $\nabla L(\param_{t}) \approx \nabla L(\paramstar) = \nabla L$.
The \sgd{} update is $\param_{t+1} = \param_{t} - \eta(\nabla L + \v{\xi}_{t})$, where $\v{\xi}_{t}=\v{g}_{B_{t}}(\param_{t})-\nabla L(\param_{t})$ satisfies $\E{\v{\xi}_{t}}=\v{0}$.
Therefore,
\begin{equation}\label{eq:stage1_sgd}
  \E{y_{t+1} - y_{t}}
  = \langle \nabla S, \E{\param_{t+1}-\param_{t}} \rangle 
  = -\eta \langle \nabla L, \nabla S \rangle
  = \eta\alpha.
\end{equation}
The \emph{expected} rate of progressive sharpening is identical to \gd.

\medskip \noindent \textit{Stage 2: Blowup under SGD.}
The displacement $x_{t}$ evolves as:
\begin{equation}\label{eq:x_update_sgd}
  x_{t+1}
  = \langle  \v{u}, \param_{t+1}-\paramstar\rangle
  = x_{t} - \eta \langle  \v{u}, \v{g}_{B_{t}}(\param_{t}) \rangle = x_{t} - \eta \langle  \v{u}, \nabla  L(\param_t) \rangle - \eta\zeta_t.
\end{equation}
where $\zeta_t = \langle  \v{u}, \v{\xi}_t\rangle$ represents the noise projected onto the top eigenvector of the Hessian.
To evaluate $\langle  \v{u}, \nabla  L(\param_t) \rangle$, we Taylor-expand $\nabla  L$ around $\paramstar$.
To leading order (see Appendix~\ref{app:proofs} for the full expansion),
\begin{equation}\label{eq:taylor_g}
  \langle  \v{u},  \nabla  L(\param_t)\rangle
  \approx  S(\paramstar) x_{t} + y_{t} x_{t} = \left(\frac{2}{\eta} + y_t\right)x_t,
\end{equation}
where we used $ S(\paramstar)=2/\eta$.
Substituting~\eqref{eq:taylor_g} into~\eqref{eq:x_update_sgd}:
\begin{equation}\label{eq:x_dynamics_sgd}
  x_{t+1}
  = -(1+\eta y_{t}) x_{t} - \eta \zeta_{t}.
\end{equation}
To recover the heuristic dynamic of \gd{} in \cite{damian2023selfstab}, we only need to set the additive noise term $-\eta\zeta_{t}$ to zero.
When $y_{t}>0$, the deterministic part $-(1+\eta y_{t})x_{t}$ causes $|x_{t}|$ to grow, but the noise $\zeta_{t}$ continuously perturbs the trajectory.

\medskip \noindent \textit{Stage~3 (Self-Stabilization under SGD).} Similar to \cite{damian2023selfstab}, expanding the cubic term in the Taylor expansion gives
\[
    \frac{1}{2}\nabla^3L(\param_t - \param^\star, \param_t - \param^\star) \approx \frac{x_t^2}{2}\nabla^3L( \v{u}, \v{u}) = \frac{x_t^2}{2}\nabla S
\]
When $x_t^2$ is large, the additional cubic term causes self-stabilization effect
\[
\E{y_{t+1}-y_{t}}= \eta\alpha + \left\langle \nabla S, \frac{x_t^2}{2}\nabla S\right\rangle =  \eta\left(\alpha-\frac{\beta}{2}\E{x_{t}^{2}}\right)
\]
However, $x_t$ is random with noise-inflated second moment. Thus, this term should have a magnitude that grows with the noise scale.

\medskip \noindent \textit{Stage~4 (Role of Noise Variance).} In this stage, $x_t$ grows to a large enough magnitude that drives the sharpness back to stability condition. In particular, the dynamic of $x_t$ is governed by 
\begin{equation}\label{eq:heur_noise_effect}
\E{x_{t+1}^{2}}=\E{(1+\eta y_{t})^{2}x_{t}^{2}}+\eta^{2}\sigma_{ \v{u}}^{2}
\end{equation}
Apply a mean-field approximation that decouples the randomness in $x_t$ and $y_t$ and solve for the sharpness at equilibrium gives
\begin{equation}\label{eq:eq_sharpness}
  \bar{y}=-\frac{\eta\beta\sigma_{ \v{u}}^{2}}{4\alpha},
  \qquad
  \bar{ S}\approx\frac{2}{\eta}-\frac{\eta\beta\sigma_{ \v{u}}^{2}}{4\alpha}
  =\frac{2}{\eta}-\Delta S.
\end{equation}
The gradient noise projection $\zeta_t$ introduces an additional growth to the sequence $\E{x_t^2}$. In order to stabilize,  the $(1 + \eta y_t)$ has to be smaller than $1$ to contract the magnitude of the first term in (\ref{eq:heur_noise_effect}). Thus, to maintain equilibrium, the system shifts to $\bar{y}<0$; and $\Delta S\propto\sigma_{ \v{u}}^{2}$, resulting in a sharpness gap.

% ════════════════════════════════════════════════════════════
\section{Main Theoretical Results}\label{sec:main_results}
% ════════════════════════════════════════════════════════════

In this section, we derive a rigorous approximation of the heuristic dynamic in Section~\ref{sec:mechanism}. Similar to \cite{damian2023selfstab}, our analysis also relies on a reference trajectory induced by the projected gradient descent defined in (\ref{eq:calM}). We make the following definitions along the reference trajectory.

\begin{definition}[Taylor Expansion Quantities at $\paramdagger_t$]\label{def:taylor_quant}~
\begin{gather*}
  \nabla L_t := \nabla L(\paramdagger_t),
  \quad
   \mat{H}_t := \nabla^{2} L(\paramdagger_t),
  \quad
   \v{u}_t :=  \v{u}(\paramdagger_t), \\
  \nabla S_t := \nabla S(\paramdagger_t),
  \quad
  \nabla S_t^{\perp} := P_{ \v{u}_t}^{\perp}\nabla S(\paramdagger_t),
  \quad
  \kappa_t := \ip{\nabla S_t}{ \v{u}_t}.
\end{gather*}
\end{definition}

\begin{definition}[Landscape Parameters]\label{def:alpha_beta_epsilon}
Define
\begin{align*}
  \alpha_t &:= -\ip{\nabla L_t}{\nabla S_t}, &
  \beta_t &:= \norm{\nabla S_t^{\perp}}^{2}, &
  \delta_t &:= \sqrt{2\alpha_t/\beta_t}, \\
  \epsilon_t &:= \eta\sqrt{\alpha_t}, &
  \delta &:= \sup_t\delta_t, &
  \epsilon &:= \sup_t\epsilon_t.
\end{align*}
\end{definition}

The coefficient $\alpha_t$ is the \emph{progressive sharpening force}: when $\alpha_t>0$, unconstrained gradient descent at $\paramdagger_t$ would increase the sharpness.
The coefficient $\beta_t$ is the \emph{strength of the self-stabilization force}: larger $\beta_t$ means stronger cubic restoring.
The scale $\delta_t$ sets the amplitude of the oscillations and serves as the fixed-point value of $|x_t|$ in the simplified dynamics.

\begin{definition}[Sharpness Propagation Factor]\label{def:beta_st}
Define $\mat{A}_k := (I-\eta \mat{H}_k)P_{ \v{u}_k}^{\perp}$ and
\begin{equation}\label{eq:beta_st}
  \beta_{s\rightarrow  t} := (\nabla S_{t+1}^{\perp})^{\top}\left[\prod_{k=t}^{s+1}\mat{A}_k\right]\nabla S_s^{\perp}.
\end{equation}
\end{definition}

The factor $\beta_{s\rightarrow  t}$ has a concrete interpretation: it measures the change in sharpness between times $s$ and $t+1$ resulting from a displacement of $\nabla S_s^{\perp}$ at time $s$.
At time $s$, a displacement in the $\nabla S_s^{\perp}$ direction increases the sharpness.
This displacement then evolves under the linearized dynamics, being multiplied by $(I-\eta \mat{H}_k)P_{ \v{u}_k}^{\perp}$ at each step $k$.
After $t-s$ steps, the resulting displacement is $[\prod_{k=t}^{s+1}\mat{A}_k]\nabla S_s^{\perp}$, and its effect on the sharpness at time $t+1$ is measured by the inner product with $\nabla S_{t+1}^{\perp}$.
When the landscape quantities are constant (as in the heuristic Section~\ref{sec:mechanism}), $\beta_{s\rightarrow  t}\equiv\beta$. 
Our theoretical result depend on the following assumptions.
\begin{assumption}[Eigengap]\label{ass:eigengap}
For some constant $c_{\mathrm{gap}}<2$, $\ltwo( \mat{H}_t) < c_{\mathrm{gap}}/\eta$ for all $t\leq\mathscr{T}$.
\end{assumption}

\begin{assumption}[Higher-Order Derivative Regularity]\label{ass:sample_reg}
Let $\rho_3,\rho_4$ be the minimum constant such that there exists a neighborhood $\mathcal{N}$ of $\{\paramdagger_t\}_{t\le\mathscr{T}}$ where every sample loss $\ell_i$ is $C^{4}$ on $\mathcal{N}$, and
\begin{itemize}
    \item $\norm{\nabla^{3}L(\param)}_{\op}\le\rho_{3}$ for all $\param\in\mathcal{N}$ and $i$.
    \item $\norm{\nabla^{3}L(\param)-\nabla^{3}L(\param')}_{\op} \leq \rho_{4}\norm{\param-\param'}$ for all $\param,\param'\in\mathcal{N}$ and $i$.
\end{itemize}
\end{assumption}

\begin{assumption}[Progressive Sharpening]\label{ass:prog_sharp}
There exists some constant $c_{\alpha},c_{\beta}\geq 0$ such that for all $t\leq\mathscr{T}$,
$\alpha_t\ge -c_{\alpha}\norm{\nabla L_t}\norm{\nabla S_t^{\perp}}$, and $\norm{\nabla S_t^{\perp}}\ge c_{\beta}\rho_3$.
\end{assumption}

\begin{assumption}[Non-Worst-Case Bounds]\label{ass:non_worst}
For all $t\leq\mathscr{T}$ and all $\v{v},\v{w}\perp \v{u}_t$:
\[
  \frac{\norm{\nabla^{3} L_t(\v{v},\v{w})}}{\rho_3\norm{\v{v}}\norm{\v{w}}}
  \le C_{\mathrm{nw}}\epsilon,
  \qquad
  \frac{|\v{v}^{\top} \mat{H}_t\v{v}|}{\norm{ \mat{H}_t}\norm{\v{v}}^{2}}
  \le C_{\mathrm{nw}}\epsilon,
  \qquad
  \frac{|\lambda_{\min}( \mat{H}_t)|}{\norm{ \mat{H}_t}}
  \le C_{\mathrm{nw}}\epsilon.
\]
\end{assumption}

\begin{assumption}[Noise Projection Non-Degeneracy]\label{ass:nondegen}
$\sigma_{ \v{u},t}^{2}>0$ for all $t$.
\end{assumption}

\begin{assumption}[Scaling conditions.]\label{ass:scaling}
We assume that $\mathscr{T}\le \frac{c_{\mathscr{T}}}{\epsilon}$ for some constant $c_{\mathscr{T}}$, $\E{\|\v{\xi}_t\|^{2}}\leq O\left(\frac{\epsilon\delta^2}{\eta^2}\right)$, and $\rho_4\leq O(\eta\rho_3^2)$.
\end{assumption}
Here, Assumption~\ref{ass:eigengap}-\ref{ass:non_worst} are inherited from \citet{damian2023selfstab}. Assumption~\ref{ass:nondegen} assumes that the gradient noise vector has a non-vanishing component when projected onto the top eigenvector of the Hessian. In later analysis we will see that this quantity determines the sharpness gap, and Assumption~\ref{ass:nondegen} is made to make sure that the sharpness gap is non-trivial. Assumption~\ref{ass:scaling} contains requirements on $\mathscr{T}, \rho_4$, which are inherited from \citet{damian2023selfstab}. Furthermore, it also requires an upper bound on the noise variance $\E{\|\v{\xi}_t\|^{2}}\leq O\left(\frac{\epsilon\delta^2}{\eta^2}\right)$. This is to make sure that the noise scale does not drive the dynamic of the sharpness into the chaotic regime where the training could potentially blow-up. With these assumptions, we are ready to state our theoretical results.

\subsection{Stochastic Coupling Theorem}\label{sec:stoch_pred}
We begin by defining the displacement from the constrained trajectory is $\v{v}_t := \param_t - \paramdagger_t$, with projections:
\begin{equation}\label{eq:xt_yt_def}
  x_t := \ip{ \v{u}_t}{\v{v}_t},
  \qquad
  y_t := \ip{\nabla S_t^{\perp}}{\v{v}_t}.
\end{equation}
The dynamic of $\v{t}$, and thus $x_t$ and $y_t$, during training are chaotic and hard to analyze. To cope with this difficulty, we define a predicted dynamic $\hat{v}_t$ and correspondingly $\hat{x}_t$ and $\hat{y}_t$ that approximate $\v{v}_{t}$ and $x_t,y_t$ reasonably well, but at the same time enjoys easier interpretability. 
\begin{definition}[Enhanced Stochastic Predicted Dynamics]\label{def:stoch_pred}
Define $\hat{\v{v}}_0 = \v{v}_0$, and for $t\geq 0$:
\begin{align}
  \hat{\v{v}}_{t+1}
  &= P_{ \v{u}_{t+1}}^{\perp}\left[(I-\eta \mat{H}_t)P_{ \v{u}_t}^{\perp}\hat{\v{v}}_t
    + \eta\nabla S_t^{\perp}\tfrac{\delta_t^{2}-\hat{x}_t^{2}}{2}\right]
  - \left((1+\eta\hat{y}_t)\hat{x}_t+\tfrac{\eta}{2}\kappa_t\hat{x}_t^{2}\right) \v{u}_{t+1}
  - \eta\v{\xi}_t,
  \label{eq:stoch_pred_full}
\end{align}
where $\hat{x}_t:=\ip{ \v{u}_t}{\hat{\v{v}}_t}$ and $\hat{y}_t:=\ip{\nabla S_t^{\perp}}{\hat{\v{v}}_t}$.
\end{definition}

\begin{definition}[Deterministic and Stochastic Step Map] The deterministic part of the step map, namely $\estep_t(\v{v})$, is
\begin{align*}
  P_{ \v{u}_{t+1}}^{\perp}\estep_t(\v{v})
  &= P_{ \v{u}_{t+1}}^{\perp}\left[\mat{A}_t\v{v}+\eta\nabla S_t^{\perp}\tfrac{\delta_t^{2}-x^{2}}{2}\right],\\
  \ip{ \v{u}_{t+1}}{\estep_t(\v{v})}
  &= -(1+\eta y)x-\tfrac{\eta}{2}\kappa_t x^{2},
\end{align*}
where $x=\ip{ \v{u}_t}{\v{v}}$, $y=\ip{\nabla S_t^{\perp}}{\v{v}}$.
The stochastic step map is $\estep_t^{S}(\v{v}):=\estep_t(\v{v})-\eta\v{\xi}_t$.
\end{definition}

\begin{lemma}[Projected Dynamics for $\hat{x},\hat{y}$]\label{lem:stoch_pred_xy}
Define $\mathcal{I}_{t+1}:=(\nabla S_{t+1}^{\perp})^{\top}[\prod_{k=t}^{0}\mat{A}_k]\hat{\v{v}}_0^{\perp}$.
Then:
\begin{align}
  \hat{x}_{t+1} &= -(1+\eta\hat{y}_t)\hat{x}_t - \tfrac{\eta}{2}\kappa_t\hat{x}_t^{2} - \eta\zeta_t,
  \label{eq:xhat_update}\\
  \hat{y}_{t+1} &= \mathcal{I}_{t+1} + \eta\sum_{s=0}^{t}\beta_{s\rightarrow  t}\frac{\delta_s^{2}-\hat{x}_s^{2}}{2}
  - \eta\sum_{s=0}^{t}\langle \v{\gamma}_{s\rightarrow  t},\v{\xi}_s\rangle,
  \label{eq:yhat_update}
\end{align}
where $\v{\gamma}_{s\rightarrow  t}:=(\nabla S_{t+1}^{\perp})^{\top}[\prod_{k=t}^{s+1}\mat{A}_k]P_{ \v{u}_{s+1}}^{\perp}$.
\end{lemma}
See Appendix~\ref{app:pred_xy_proof} for the full proof. Our main theoretical result is a coupling theorem that shows a strong dependence of the loss and sharpness dynamic on the predicted dynamic we defined above. Establishing this result will allow us to derive the explicit sharpness gap in later sections.

\begin{theorem}[Stochastic Coupling in Mean Square]\label{thm:coupling}
Let Assumption~\ref{ass:eigengap}-\ref{ass:scaling} hold with a small enough constant $c_{\mathscr{T}}$. If $\norm{\v{v}_0}^2 \leq c_0\delta^2$ for some small enough constant $c_0$, then for any small constant $\omega$ such that $\omega^{-1}\leq O(1)$, we have that with probability at least $1 - \omega$, the following holds for all $t\leq \mathscr{T}$:
\begin{align*}
    & \max\left\{\norm{\v{v}_t},\norm{\hat{\v{v}}_t}\right\} \leq O(\delta), \tag{Norm Bound}\label{eq:norm_bound}\\
  & \norm{\v{v}_t-\hat{\v{v}}_t}
  \leq O(\epsilon\delta),
  \tag{Deviation}
  \label{eq:coupling_dev}\\
  &  L(\param_t) =  L(\paramdagger_t)+\frac{1}{\eta}\hat{x}_t^{2}+ O\left(\frac{\epsilon\delta^{2}}{\eta}\right),
  \tag{Loss}
  \label{eq:coupling_loss}\\
  &  S(\param_t) = 
   \frac{2}{\eta} + \hat{y}_t + \kappa_t\hat{x}_t + O\left(\frac{\epsilon^{2}}{\eta}\right).
  \tag{Sharpness}
  \label{eq:coupling_sharp}
\end{align*}
\end{theorem}
Full details are in Appendix~\ref{app:coupling_proof}. In particular, the theorem states three results. First, both the true deviation from constrained trajectory, namely $\v{v}_t = \param_t - \paramdagger_t$, and the predicted dynamic $\hat{\v{v}}_t$ has a norm upper bounded by $O(\delta)$. Second, the difference between $\v{v}_t$ and $\hat{\v{v}}_t$ is upper bounded by $O(\epsilon\delta)$, showing a relatively small difference compared to the magnitude of the two vectors $O(\delta)$. This shows that the predicted dynamic roughly describes the true deviation from the constrained trajectory. Moreover, the theorem shows that the loss of the SGD training deviates from the loss of the constrained loss by a magnitude of $\frac{1}{\eta}\hat{x}_t^2$ with an additional error of $O\left(\frac{\epsilon\delta^2}{\eta}\right)$. As $\hat{x}_t$ roughly scales with $O(\delta)$, the error term $O\left(\frac{\epsilon\delta^2}{\eta}\right)$ has a small influence in the loss dynamic. Combined with Lemma~\ref{lem:dagger_step}, this shows a decrease of the SGD loss with an oscillation cased by the instability term $\hat{x}_t$. Lastly, we can approximate the SGD trajectory's sharpness by a combination of $\hat{y}_t$ and $\hat{x}_t$. In the next section, we will demonstrate how this approximation of the sharpness leads to a sharpness gap that depends on the gradient noise variance.

\subsection{Derivation of the Sharpness Gap}\label{sec:sharpness_gap}
To derive a meaningful closed-form sharpness gap, we need make certain simplification to avoid excessive dependencies. In particular, we work in a frozen-coefficient regime with an explicit closure assumption, and assumes a mean-field approximation of the coupling between $\hat{y}_t$ and $\hat{x}_t$, as stated by the assumptions below.

\begin{assumption}[Stationarity and Decorrelation]\label{ass:decorr}
The frozen-coefficient stochastic predicted dynamics admit a stationary distribution for $(\hat{x}_t,\hat{y}_t)$.
At stationarity:
\[
  \left|\E{(1+\eta\hat{y}_t)^{2}\hat{x}_t^{2}}
   -(1+\eta\bar{y})^{2}\E{\hat{x}_t^{2}}\right|
  \le O(\epsilon^{2}\delta^{2}),
  \qquad \bar{y}:=\E{\hat{y}_t}.
\]
\end{assumption}

\begin{assumption}[Frozen Coefficients]\label{ass:frozen}
Over the equilibrium window $[T_0, T_0+\mathscr{T}]$,
the following quantities evaluated along the reference
trajectory $\paramdagger_t$ are approximately constant.
Specifically:
\begin{enumerate}[label=(\roman*)]
\item \emph{Landscape geometry:}
The Hessian $ \mat{H}_t\equiv \mat{H}$,
the top eigenvector $ \v{u}_t\equiv \v{u}$,
and the sharpness gradient
$\nabla S_t^{\perp}\equiv\nabla S^{\perp}$
are constant.
As a consequence, the derived scalar quantities are also
constant:
$\alpha_t\equiv\alpha>0$,
$\beta_t\equiv\beta>0$,
$\delta_t\equiv\delta=\sqrt{2\alpha/\beta}$,
$\kappa_t\equiv\kappa  = \langle \nabla S,  \v{u}\rangle \le O(\epsilon\rho_3)$.
\item \emph{Sharpness propagation:}
$\nabla S^{\perp}\in\ker( \mat{H}|_{ \v{u}^{\perp}})$.
\end{enumerate}
\end{assumption}
Assumption~\ref{ass:decorr} decouples the expectation of $(1 + \eta\hat{y}_t)^2\hat{x}_t^2$ into expectation of $\hat{y}_t$ and $\hat{x}_t^2$. Intuitively, $\hat{y}_t$ and $\hat{x}_t$ are the projection of $\hat{v}_t$ onto orthogonal directions, so the statistical correlation between the two are small. This is verified experimentally in Appendix~\ref{app:mean-field-exp}. The set of conditions in Assumption~\ref{ass:frozen} are made so that the dynamic of $\hat{x}_t$ and $\hat{y}_t$ can be well understood. Assuming fixed values of $\mat{H},\alpha,\beta,\delta$, and $\kappa$ allows us to express the sharpness gap as a fixed value instead of depending on the complicated aggregation of these quantities. Moreover, the assumption that $\kappa$ is small and $\nabla S^\perp \in \ker( \mat{H}|_{ \v{u}^{\perp}})$ are simplification of the real-world scenario where high dimensionality implies near-orthogonality between vectors, and the neural network Hessian are effectively low-rank, especially after entering the EoS regime.

With the assumptions above, we present the theorem that characterizes the sharpness gap induced by SGD. The predicted dynamics are at equilibrium when the joint
distribution of $(\hat{x}_t, \hat{y}_t)$ is at stationary, and thus so are the statistics of the joint distribution. In this section, we will focus on the following statistics:
\[
    \E{\hat{y}_{t+1}} = \E{\hat{y}_t},\;\;\text{(mean stationarity of } \hat{y}\text{)};\quad \E{\hat{x}_{t+1}^{2}} = \E{\hat{x}_{t}^{2}},\;\;\text{(second-moment stationarity of }\hat{x}\text{)}
\]
These two conditions yield a system of two equations that
determine $\E{\hat{x}_t^{2}}$ and $\bar{y} = \E{\hat{y}_t}$.
The mean stationarity of $\hat{y}$ naturally involves
$\E{\hat{x}_t^{2}}$ (because the self-stabilization force is
proportional to $x_t^{2}$), while the second-moment stationarity
of $\hat{x}$ is where the noise variance $\sigma_{ \v{u}}^{2}$
enters (because the noise is mean-zero, it is invisible to
first moments but leaves a trace in second moments).

\begin{theorem}[Equilibrium Sharpness under SGD]\label{thm:sharp_gap}
Under Assumptions~\ref{ass:eigengap}--\ref{ass:scaling},\ref{ass:decorr}, and \ref{ass:frozen}, with probability at least 1 - $\omega$, at stationary distribution of the predicted dynamic $\hat{y}_t$ and $\hat{x}_t$, we have that:
\[
    \E{\hat{x}_t^2} = \delta^2;\quad \E{\hat{y}_t} = -\frac{\eta\beta\sigma_{ \v{u}}^{2}}{4\alpha}
  + O \left(\frac{\epsilon^{2}}{\eta}\right)
\]
and the sharpness at stability is given by
\begin{equation}\label{eq:sharp_gap}
  \E{ S(\param_t)}
  = \frac{2}{\eta} - \frac{\eta\beta\sigma_{ \v{u}}^{2}}{4\alpha}
  + O \left(\frac{\epsilon^{2}}{\eta}\right).
\end{equation}
\end{theorem}
The proof is provided in Appendix~\ref{app:gap_proof}. Theorem~\ref{thm:sharp_gap} states that, at the stationary distribution of $\hat{x}_t$ and $\hat{y}_t$, the unstable force $\hat{x}_t$ has a second moment that matches the stabilizing force $\delta^2$. This recovers the same condition in \citet{damian2023selfstab}. In the meantime, the sharpness deviation from $\frac{2}{\eta}$, namely $\hat{y}_t$, has a mean at $-\frac{\eta\beta\sigma_{ \v{u}}^{2}}{4\alpha} + O \left(\frac{\epsilon^{2}}{\eta}\right)$. Ignoring the small error term $O \left(\frac{\epsilon^{2}}{\eta}\right)$, we derive a sharpness gap of $-\frac{\eta\beta\sigma_{ \v{u}}^{2}}{4\alpha}$, showcasing that the supressed sharpness under the injection of gradient noise. Noticeably, this sharpness gap depends linearly on the variance of the noise when projected onto the top Hessian eigenvector $\sigma_{\v{u}}$, which matches the prior work that a smaller batch size result in a less sharp solution \cite{cohen2021eos,andreyev2025eoss}.

\begin{figure}[t]
    \centering
    \includegraphics[width=\textwidth]{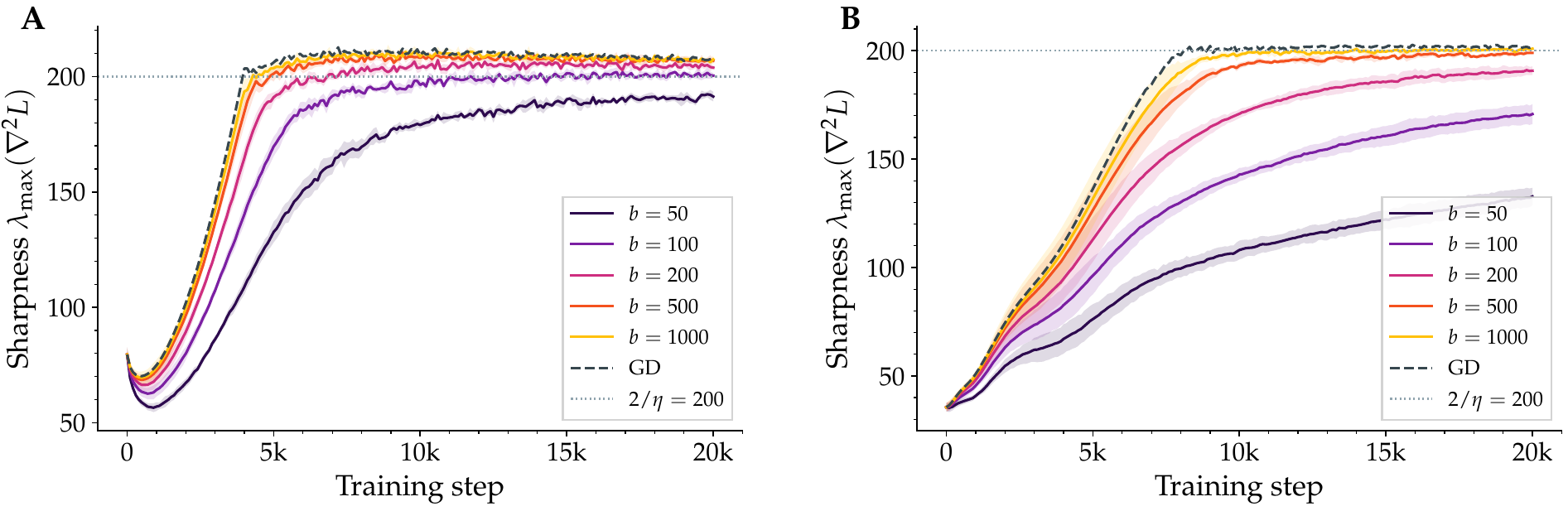}
    \caption{Sharpness $\lambda_{\max}(\nabla^2 L)$ during training
    for different batch sizes.
    \textbf{(A)}~FC-Tanh + MSE: smaller batch sizes produce
    markedly lower equilibrium sharpness. The dashed black curve
    shows full-batch GD.
    \textbf{(B)}~CNN + MSE: the same pattern holds with a
    convolutional architecture, with an even larger gap
    ($S_{\mathrm{eq}} \approx 131$ at $b=50$ vs.\
    $S_{\mathrm{GD}} \approx 202$).
    Shaded regions indicate $\pm 1$ standard deviation across 5 seeds.}
    \label{fig:trajectories}
\end{figure}

\section{Experimental Validation}\label{sec:experiments}
\begin{wrapfigure}{r}{0.67\linewidth}
  \centering
  \includegraphics[width=\linewidth]{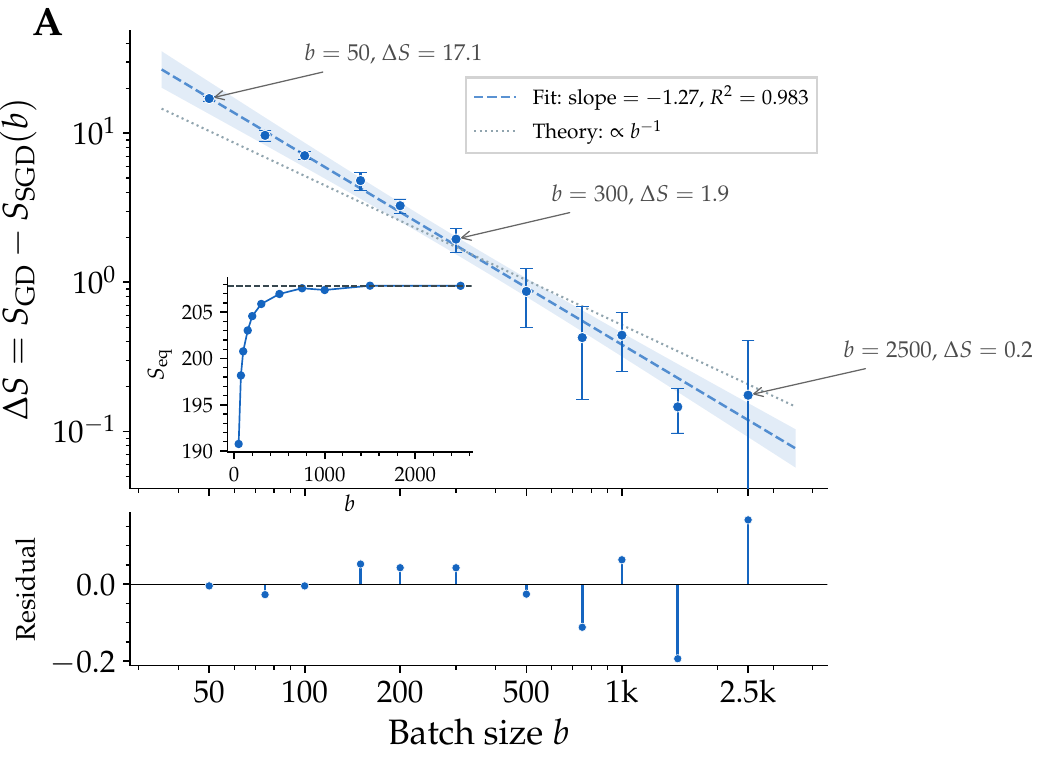}
    \caption{Log-log plot of the sharpness gap $\Delta S$ versus
    batch size~$b$ for FC-Tanh.
    Fitted slope $= -1.27$ (theory: $-1$), $R^2 = 0.98$.
    The dotted line shows the theoretical $\Delta S \propto 1/b$
    reference slope. The shaded band indicates the 95\% confidence
    interval on the fit. Error bars indicate $\pm 1$ standard
    deviation across 3--5 seeds. The inset shows raw equilibrium
    sharpness $S_{\mathrm{eq}}$ versus $b$. The lower panel
    shows residuals from the log-log fit.}
    \label{fig:scaling}\vspace{-0.5cm}
\end{wrapfigure}

We validate the predictions of our theoretical analysis through comprehensive
experiments on neural networks trained with SGD at the edge of stability.
Our experiments are designed to test two specific predictions: \vspace{-0.2cm}
\begin{enumerate}[leftmargin=*]
    \item \textbf{Batch-size-dependent sharpness gap:} SGD with batch size~$b$
    equilibrates at sharpness $S_{\mathrm{SGD}}(b) < S_{\mathrm{GD}}$,
    with $\Delta S = S_{\mathrm{GD}} - S_{\mathrm{SGD}}(b) \propto 1/b$. \vspace{-0.2cm}
    \item \textbf{Noise variance scaling:} The projected gradient noise
    $\sigma_u^2 = u^\top \Sigma_b u$ scales as $\sigma_u^2 \propto 1/b$,
    providing the mechanism through which batch size modulates sharpness.
\end{enumerate}

Our primary setup follows \cite{cohen2021eos}: a two-hidden-layer FC-Tanh network trained with MSE loss on CIFAR-10, with additional experiments on a CNN and FC-ReLU network (with both MSE and cross-entropy losses) to verify generality. We train on 5{,}000 CIFAR-10 images using vanilla SGD with default learning rate $\eta = 0.01$ (stability threshold $2/\eta = 200$), varying $\eta \in \{0.005, 0.008, 0.01, 0.015, 0.02\}$ for scaling experiments, and repeat each run over 3--5 random seeds. Throughout training we periodically compute the top Hessian eigenvalue, projected gradient noise variance, batch sharpness, and the sharpness gradient quantities $\alpha$ and $\beta$\footnote{Notice that in our theory these quantities are measured along the constrained trajectory. In experiments we use the approximation of the measurements along the SGD trajectory.}; equilibrium sharpness is taken as the average of $\lambda_{\max}$ over the final 20 measurements.

\subsection{Batch-Size-Dependent Sharpness Gap}\label{sec:exp-gap}
\textbf{Sharpness trajectories.}
Figure~\ref{fig:trajectories} shows the sharpness trajectories for representative batch sizes. All methods exhibit progressive sharpening: the top eigenvalue increases monotonically during early training before reaching a plateau. Full-batch GD equilibrates at $S_{\mathrm{GD}} \approx 207.8$, slightly above $2/\eta = 200$ due to the oscillation asymmetry identified by \citet{damian2023selfstab}. Critically, SGD with smaller batch sizes equilibrates at \emph{strictly lower} sharpness: $S_{\mathrm{eq}} \approx 190.8$ for $b = 50$ versus $S_{\mathrm{eq}} \approx 207.4$ for $b = 1000$. This $17$-point gap between $b = 50$ and GD is the stochastic self-stabilization effect predicted by our theory.

The CNN architecture with MSE loss (Figure~\ref{fig:trajectories}B) exhibits the same qualitative pattern with a dramatically amplified effect: $S_{\mathrm{eq}} \approx 131$ at $b = 50$ versus $S_{\mathrm{GD}} \approx 202$, yielding a gap of $\Delta S \approx 71$. This is possibly due to the stronger gradient noise of CNN along the top Hessian eigenvector.

\medskip \noindent \textbf{Scaling with batch size.}
Figure~\ref{fig:scaling} plots the sharpness gap $\Delta S = S_{\mathrm{GD}} - S_{\mathrm{SGD}}(b)$ against batch size on a log-log scale. The theory predicts $\Delta S \propto 1/b$, corresponding to a slope of $-1$ in log-log coordinates. We observe a fitted slope of $-1.27$ ($R^2 = 0.98$, $p < 10^{-5}$), in good agreement with the theoretical prediction. The slightly steeper slope may reflect higher-order corrections or the fact that $\beta$ and $\alpha$ have mild batch-size dependence not captured by the leading-order formula. The residual panel confirms that the power-law fit captures the data well, with no systematic deviations across the range of batch sizes tested.

\begin{wrapfigure}{r}{0.5\textwidth}
    \centering
    \includegraphics[width=\linewidth]{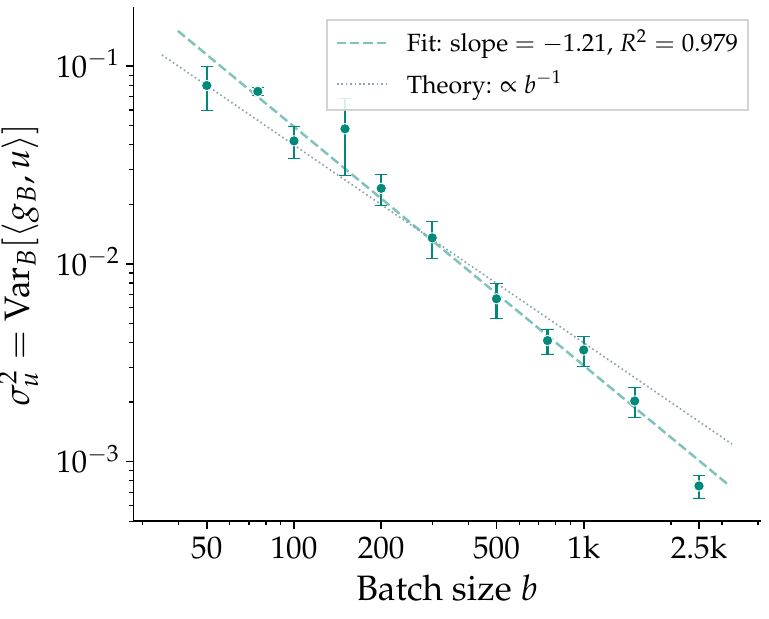}
    \caption{Projected gradient noise variance $\sigma_u^2 =
    \mathrm{Var}_B[\langle g_B, u\rangle]$ versus batch size
    on a log-log scale. Fitted slope $= -1.21$ (theory: $-1$),
    $R^2 = 0.979$.}
    \label{fig:noise}\vspace{-1cm}
\end{wrapfigure}

The CNN architecture with MSE loss exhibits an even stronger effect: the sharpness gap at $b = 50$ is $\Delta S \approx 71$ (versus $\Delta S \approx 17$ for FC-Tanh), with GD reaching $S_{\mathrm{GD}} \approx 202$. The larger gap likely reflects the CNN's richer parameter space and stronger gradient noise along the top eigenvector. The FC-ReLU architecture shows a comparable pattern to FC-Tanh ($\Delta S \approx 16$ at $b=50$), confirming that the mechanism is not activation-specific.

\subsection{Noise Variance Scaling}\label{sec:exp-noise}
The formula $\Delta S = \eta \beta \sigma_u^2 / (4\alpha)$ predicts that the sharpness gap is mediated by the projected gradient noise $\sigma_u^2 = u^\top \Sigma_b u$, where $\Sigma_b$ is the covariance of the mini-batch gradient. Since mini-batch gradients are averages over $b$ independent per-sample gradients, $\Sigma_b = \Sigma_1 / b$ and hence $\sigma_u^2 \propto 1/b$.

Figure~\ref{fig:noise} confirms this prediction: the measured $\sigma_u^2$ scales with batch size with a fitted log-log slope of $-1.21$ ($R^2 = 0.979$), close to the theoretical $-1$. Combined with the batch-size scaling of~$\Delta S$, this establishes the causal chain: smaller batches lead to larger projected noise $\sigma_u^2$, which lead to stronger stochastic self-stabilization and, thus, lower equilibrium sharpness.

\subsection{Anatomy of the Sharpness Gap}\label{sec:exp-anatomy}

The equilibrium formula $\Delta S = \eta \beta \sigma_u^2 / (4\alpha)$
decomposes the sharpness gap into three landscape-dependent factors.
While the absolute quantitative prediction is sensitive to the
measurement of~$\alpha$ (see Section~\ref{sec:exp-landscape}), we
can test the \emph{relative} structure of the gap across batch sizes
by examining how the measured gap compares to the predicted scaling
from the landscape quantities.

Figure~\ref{fig:anatomy}A shows the measured sharpness gap
$\Delta S = S_{\mathrm{GD}} - S_{\mathrm{SGD}}(b)$ for five
batch sizes ($b \in \{50, 100, 200, 500, 1000\}$), overlaid with
the power-law fit $\Delta S \propto b^{-1.27}$ from the log-log
regression. The fit accurately captures the monotonic decrease
across the full range.

Figure~\ref{fig:anatomy}B tests whether the product
$\beta \sigma_u^2$---the numerator of the gap formula---tracks the
measured gap across batch sizes. We plot the measured $\Delta S$
alongside $\beta \sigma_u^2$ (rescaled to match at $b = 50$) for
each batch size. The close agreement confirms that the product
$\beta \sigma_u^2$ captures the batch-size dependence of the gap,
consistent with the theoretical prediction that $\alpha$ (the
denominator) varies slowly across batch sizes at fixed~$\eta$.

\begin{figure}[t]
    \centering
    \includegraphics[width=0.9\textwidth]{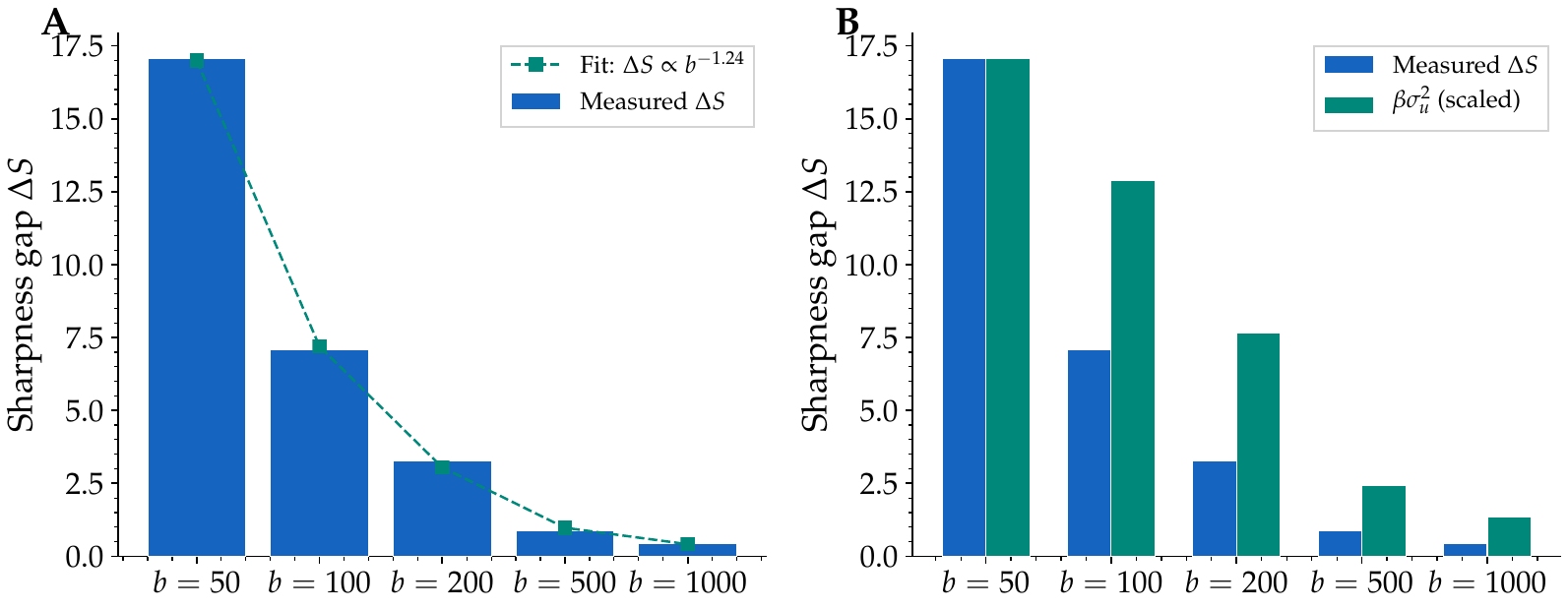}
    \caption{Anatomy of the sharpness gap.
    \textbf{(A)}~Measured $\Delta S$ across batch sizes with the
    fitted power-law $\Delta S \propto b^{-1.27}$.
    \textbf{(B)}~Comparison of measured $\Delta S$ (blue) with the
    rescaled product $\beta\sigma_u^2$ (teal) for each batch size.
    The close tracking confirms that the gap is driven by the
    product of gradient noise and self-stabilization strength.}
    \label{fig:anatomy}
\end{figure}

\begin{figure}[t]
    \centering
    \includegraphics[width=0.9\textwidth]{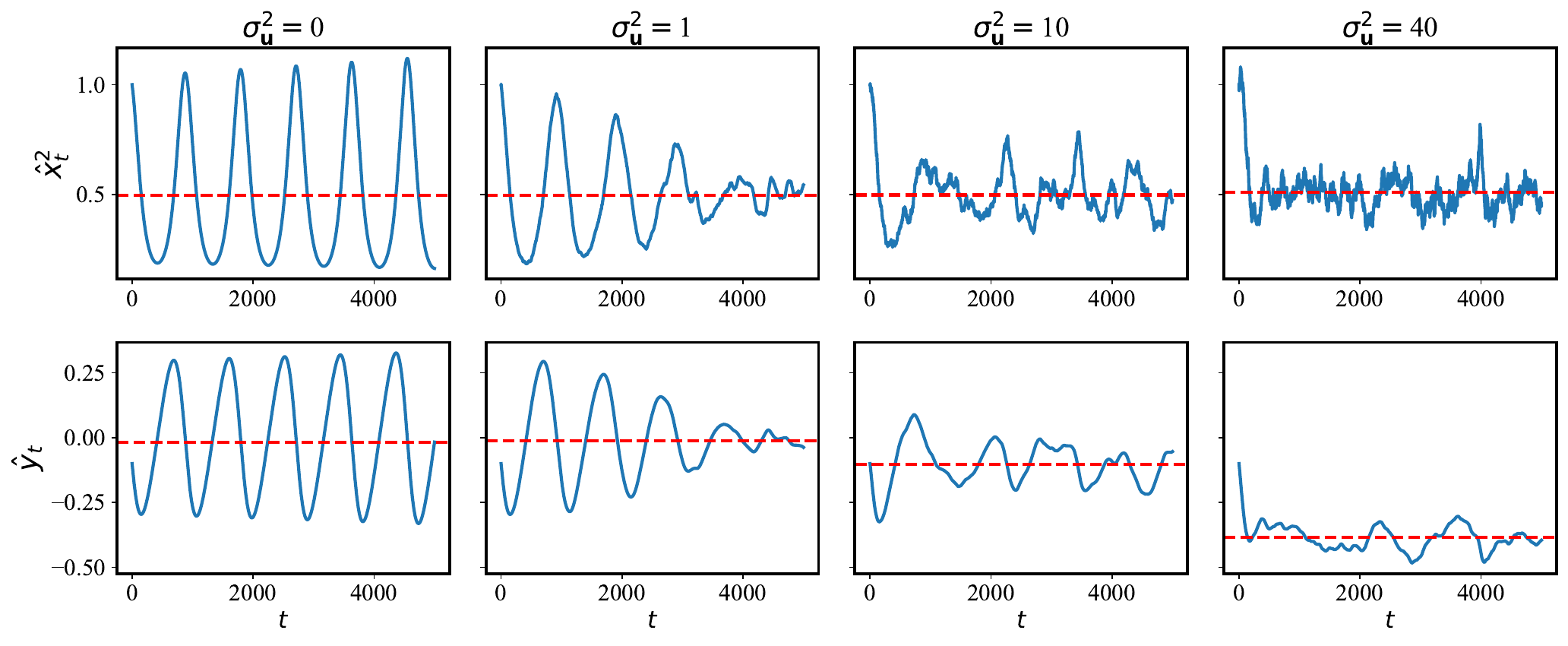}
    \caption{Simulation of the $\hat{x}_t^2,\hat{y}_t$ dynamic given in Lemma~\ref{lem:stoch_pred_xy}. We choose $\beta_{s\rightarrow t} = \beta = 1$, and $\delta_t^2 = \delta^2 = 0.5$. Also, we choose $\eta = 0.01, \kappa_t = 0, \v{\gamma}_{s\rightarrow t} = \bm{0}$, and $\zeta_t \sim\mathcal{N}(0, \sigma_{\v{u}}^2)$ for difference choice of $\sigma_{\v{u}}$. We choose initial value $\hat{x}_t = 1$ and $\hat{y}_t = 0$. Trajectories are averaged over 100 runs.}
    \label{fig:sde_dynamic}
\end{figure}

\subsection{Simulation of Predicted Dynamic}
We simulate the predicted dynamic $\hat{x}_t$ and $\hat{y}_t$ given in Lemma~\ref{lem:stoch_pred_xy}. To isolate the interesting factors, we simplify the dynamic to $\beta_{s\rightarrow t} = \beta, \delta_t = \delta$, and $\kappa_t = 0, \v{\gamma}_{s\rightarrow t} = \bm{0}$. The simulation result is shown in Figure~\ref{fig:sde_dynamic}, where the red dashed line shows the approximation of the stationary condition by taking a time-average of trajectory. Each column in Figure~\ref{fig:sde_dynamic} shows the dynamic of a difference $\sigma_{\bfu}^2$ value, ranging from $0$ (deterministic case) to $40$. The $\hat{x}_t^2,\hat{y}_t$ dynamic in the first column reproduces the behavior in \citet{damian2023selfstab}. Observing the first row that demonstrate the $\hat{x}_t^2$ dynamic, one could see that increasing the noise suppresses the oscillation scale of $\hat{x}_t^2$, but keeping the mean $\E{\hat{x}_t^2}$ at $\delta^2 = 0.5$ and breaking the periodic structure. In the second row, one could observe that the mean at stationary (dashed red line) gradually shifts down as $\sigma_{\v{u}}^2$ grows. In the mean time, the scale of the oscillation seems to also be suppressed. Together, Figure~\ref{fig:sde_dynamic} validates the theoretical result in Theorem~\ref{thm:sharp_gap}.

\section{Related Work}\label{sec:related}

\medskip\noindent\textbf{Edge of Stability.}
The \emph{Edge of Stability} (\eos) phenomenon was systematically documented by \citet{cohen2021eos}, who observed that full-batch gradient descent on neural networks operates in a regime where the sharpness $S(\param) = \lmax(\nabla^2 L(\param))$ rises steadily until it reaches $2/\eta$ and then hovers near this instability threshold, while the training loss continues to decrease non-monotonically over long timescales.
Earlier empirical work by \citet{jastrzebski2019on,jastrzebski2020breakeven} and \citet{xing2018walk} had already suggested that the sharpness of the loss landscape along the SGD trajectory is strongly influenced by the optimizer hyperparameters, particularly the learning rate and batch size.

Subsequent theoretical work has analyzed the EoS from multiple angles.
\citet{arora2022eos} studied EoS for normalized gradient descent and established conditions under which training dynamics evolve along a sharpness-minimizing flow on the manifold of minimizers.
\citet{damian2023selfstab} provided the first rigorous mechanism for GD self-stabilization via a cubic Taylor expansion of the loss, proving that the third-order term acts as a restoring force that keeps $S(\param)$ near $2/\eta$; our work directly extends this framework to the stochastic setting.
\citet{zhu2023eos} analyzed EoS dynamics through a minimalist example, deriving the sharpness behavior near the stability threshold.
\citet{li2022sharpness} analyzed sharpness dynamics in two-layer linear networks, providing further theoretical grounding for the four-phase picture of progressive sharpening and edge-of-stability behavior.
\citet{chen2023beyond} studied gradient descent behavior beyond the edge of stability via two-step gradient updates, characterizing a local implicit bias toward symmetric solutions.
\citet{kreisler2023eos} showed that the sharpness of the gradient flow solution decreases monotonically during EoS training in scalar networks.
\citet{agarwala2023second} demonstrated that second-order regression models exhibit progressive sharpening and EoS, suggesting these phenomena are not unique to neural networks.

On the SGD side, \citet{andreyev2025eoss} identified \emph{Batch Sharpness}---the expected directional curvature of the mini-batch Hessian along the mini-batch gradient---as the quantity that saturates at $2/\eta$ during SGD training, terming this the \emph{Edge of Stochastic Stability}; their work explicitly poses as an open problem the mechanism behind the suppression of full-batch sharpness below $2/\eta$, which our paper resolves.
\citet{lee2023ias} similarly introduced \emph{Interaction-Aware Sharpness} and characterized the onset of oscillatory behavior in SGD, also identifying the sharpness suppression as an open theoretical question.

\medskip\noindent\textbf{Implicit Bias of SGD, Sharpness, and the Role of Batch Size and Learning Rate.}
A foundational empirical observation is that SGD with smaller batch sizes or larger learning rates tends to find flatter minima that generalize better \citep{keskar2017largebatch,smith2017three,jastrzebski2019on}.
\citet{keskar2017largebatch} showed that large-batch SGD converges to sharper minimizers with worse generalization, while small-batch SGD consistently finds flatter minima due to the inherent noise in gradient estimation.
The relationship between batch size and learning rate in controlling sharpness was further studied by \citet{smith2017three}, who showed that the equilibrium properties of SGD depend primarily on the ratio $\eta/b$, and by \citet{wu2018sgd}, who analyzed the dynamical stability of global minima under SGD from a linear stability perspective.
\citet{smith2021origin} characterized the implicit regularizer of SGD using backward error analysis, showing that the scale of implicit regularization is proportional to $\eta/b$ and penalizes the norms of mini-batch gradients.
The alignment between SGD noise and the sharp directions of the loss landscape was studied by \citet{wu2022alignment}, who proved that linearly stable minima under SGD must satisfy a Frobenius-norm bound on the Hessian that depends on the ratio $b/\eta$.
\citet{haochen2021shape} showed that the \emph{shape} of the noise covariance---not merely its scale---determines which minimizers SGD converges to, demonstrating that parameter-dependent mini-batch noise yields a stronger implicit bias toward flat minima than spherical Gaussian noise.
\citet{damian2021labelnoise} proved that label noise SGD provably prefers flat global minimizers by penalizing sharpness through its parameter-dependent noise structure.
Taken together, these works establish the empirical and theoretical backdrop for the batch-size-dependent sharpness we explain: our closed-form sharpness gap $\Delta S = \eta\beta\sigma_{\v{u}}^2/(4\alpha)$ provides a mechanistic account of the $\eta/b$ dependence through the projected noise variance $\sigma_{\v{u}}^2 \propto 1/b$.

\medskip\noindent\textbf{Progressive Sharpening and Feature Learning.}
Progressive sharpening---the steady increase of $S(\param)$ during early training until it reaches $2/\eta$---was identified as a key phenomenon alongside EoS by \citet{cohen2021eos}.
\citet{jastrzebski2020breakeven} characterized the \emph{break-even point} on the SGD trajectory, after which the curvature of the loss surface and noise in the gradient become implicitly regularized; they showed that using a large learning rate in the early phase of training reduces gradient variance and improves the conditioning of the loss surface.
\citet{li2022sharpness} analyzed the sharpness dynamics across all four phases of the GD trajectory in two-layer linear networks, providing a theoretical proof of progressive sharpening in this setting.
\citet{agarwala2023second} showed that second-order regression models exhibit progressive sharpening to the edge of stability, with the NTK eigenvalue sharpening toward a value that can be computed explicitly.
\citet{kalra2023universal} unified all four phases of sharpness dynamics---early reduction, saturation, progressive sharpening, and EoS---through a fixed-point analysis of a two-layer linear network, also characterizing a period-doubling route to chaos as the learning rate is increased.

Progressive sharpening is intimately connected to feature learning: the sharpening arises because gradient steps in the direction of $-\nabla L(\param)$ increase $S(\param)$ when $\alpha(\param) = -\langle\nabla L(\param),\nabla S(\param)\rangle > 0$, which is in turn tied to how the network's representations evolve.
\citet{damian2022representations} demonstrated that gradient descent on two-layer networks can escape the kernel (NTK) regime and learn task-relevant representations, with the sharpening of the loss Hessian being a signature of this feature learning process.
The connection between sharpness and the transition from lazy to rich (feature-learning) regimes has been further studied in the context of $\mu$P and related parameterizations \citep{lewkowycz2020catapult,wang2022regularity}, showing that networks operating at EoS are precisely those undergoing non-trivial feature learning.

\section{Conclusion}
We have presented a theoretical explanation for why mini-batch \sgd{} self-stabilizes at sharpness values strictly below the $2/\eta$ threshold observed in full-batch gradient descent. By extending the self-stabilization framework of \citet{damian2023selfstab} to the stochastic setting, we identified gradient noise as a sharpness-suppressing force: noise projected onto the top Hessian eigenvector inflates the second moment of the oscillatory dynamics, strengthening the cubic restoring force and shifting the equilibrium sharpness downward. Our stochastic coupling theorem formalizes this mechanism, showing that \sgd{} trajectories are well-approximated by a constrained reference trajectory plus stochastic predicted dynamics. From the stationarity conditions of these dynamics, we derived the closed-form sharpness gap $\Delta S = \eta\beta\sigma_{\v{u}}^{2}/(4\alpha)$, which cleanly decomposes the gap into the learning rate, the self-stabilization strength, the projected noise variance, and the progressive sharpening rate. This formula recovers full-batch \gd{} as a special case, predicts that smaller batches yield flatter solutions, and identifies \emph{Batch Sharpness} as the quantity that genuinely saturates at $2/\eta$. Experimental validation across FC-Tanh, FC-ReLU, CNN, and ResNet architectures confirms the predicted power-law scaling of the sharpness gap with batch size, resolving an open question posed by both \citet{cohen2021eos} and \citet{andreyev2025eoss}.
\bibliography{ref}
\bibliographystyle{plainnat}

\newpage
\appendix

\section{Cubic Taylor Expansion Details}\label{app:proofs}

\begin{lemma}[Cubic expansion of the mini-batch gradient]\label{lem:cubic_expansion}
Let $\param^\star\in \mathcal{M}$ Let $\param_{t}=\paramstar+\v{v}_{t}$ with $\norm{\v{v}_{t}}\le R$ for some $R>0$.
Under Assumptions~\ref{ass:sample_reg}, the mini-batch gradient satisfies:
\begin{align}
  \nabla L(\param_{t})
  &= \nabla L(\paramstar)
     + \nabla^2 L(\paramstar) \v{v}_{t}
     + \frac{1}{2}\nabla^{3}L(\paramstar)(\v{v}_{t},\v{v}_{t})
     + \v{r}_{t},\label{eq:cubic_exp}
\end{align}
where $\norm{\v{r}_{t}} \leq O(\rho_{4})\norm{\v{v}_{t}}^{3}$.
\end{lemma}

\begin{proof}
By the integral form of Taylor's theorem applied to $\v{g}_{B_{t}}(\cdot) = \nabla L_{B_{t}}(\cdot)$ around $\paramstar$:
\begin{align*}
  \nabla L(\paramstar+\v{v})
  &= \nabla L(\paramstar)
     + \nabla^2L(\paramstar)\v{v}
     + \frac{1}{2}\nabla^3L(\paramstar)(\v{v},\v{v}) \\
  &\quad\quad\quad   + \underbrace{\int_{0}^{1}\frac{(1-s)^{2}}{2}\nabla^4 L(\paramstar+s\v{v})(\v{v},\v{v},\v{v}) \dd s}_{\v{r}_t}.
\end{align*}
By Assumption~\ref{ass:sample_reg}, we have that $\norm{\nabla^4 L(\paramstar+s\v{v})(\v{v},\v{v},\v{v})}_{\text{op}}\leq \rho_4$. Therefore, 
\[
  \norm{\v{r}_{t}}
  \le \frac{1}{6}\sup_{s\in[0,1]}\norm{\nabla^{4}L(\paramstar+s\v{v}_{t})}_{\op}\norm{\v{v}_{t}}^{3}
  \le \frac{\rho_{4}}{6}\norm{\v{v}_{t}}^{3}.
\]
\end{proof}

\section{Lipschitz Properties of Landscape Quantities}\label{app:dagger_proofs}
 
\begin{lemma}[Lipschitz Properties, {\citet{damian2023selfstab}}]\label{lem:lipschitz}
Define $\calS_t:=B(\paramdagger_t,(2-c_{\mathrm{gap}})/(4\eta\rho_3))$.
Under Assumption~\ref{ass:eigengap},\ref{ass:sample_reg}, it holds that
\begin{enumerate}[label=(\arabic*)]
\item $\param\mapsto\nabla L(\param)$ is $O(\eta^{-1})$-Lipschitz in $\calS_t$.
\item $\param\mapsto \mat{H}(\param)$ is $\rho_3$-Lipschitz.
\item $\param\mapsto\lambda_i( \mat{H}(\param))$ is $\rho_3$-Lipschitz.
\item $\param\mapsto \v{u}(\param)$ is $O(\eta\rho_3)$-Lipschitz in $\calS_t$.
\item $\param\mapsto\nabla S(\param)$ is $L_S$-Lipschitz in $\calS_t$ with $L_S:=O(\eta\rho_3^2)$.
\end{enumerate}
\end{lemma}
\begin{proof}
    The proof is identical to \citet{damian2023selfstab}, Lemma 12.
\end{proof}

\section{Proof of Lemma~\ref{lem:stoch_pred_xy}}\label{app:pred_xy_proof}
 
\begin{proof}
We  first derive the $\hat{x}$ update, then unroll the recursion for $\hat{\v{v}}^{\perp}$ to obtain the $\hat{y}$ formula.
 
\medskip\noindent\textit{Part A: Derivation of the $\hat{x}$ update.}
From Definition~\ref{def:stoch_pred}, project $\hat{\v{v}}_{t+1}$ onto $ \v{u}_{t+1}$:
\begin{align*}
  \hat{x}_{t+1} = \ip{ \v{u}_{t+1}}{\hat{\v{v}}_{t+1}} = \underbrace{\ip{ \v{u}_{t+1}}{P_{ \v{u}_{t+1}}^{\perp}[\cdots]}}_{=0\text{ by definition}}
    - (1+\eta\hat{y}_t)\hat{x}_t - \tfrac{\eta}{2}\kappa_t\hat{x}_t^{2}
    - \eta\underbrace{\ip{ \v{u}_{t+1}}{\tilde{\v{\xi}}_t}}_{\tilde{\zeta}_t}.
\end{align*}
The first term vanishes because $P_{ \v{u}_{t+1}}^{\perp}$ projects orthogonally to $ \v{u}_{t+1}$. This yields~\eqref{eq:xhat_update}.
 
\medskip\noindent\textit{Part B: Unrolling the recursion for $\hat{\v{v}}^{\perp}_{t+1}$.} Define $\hat{\v{v}}_t^{\perp}:=P_{ \v{u}_t}^{\perp}\hat{\v{v}}_t$. Then we have that
\[
    \hat{y}_{t+1} = \langle \nabla S_{t+1}^\perp, \hat{\v{v}}_{t+1}\rangle = \nabla S_{t+1}^\top P_{ \v{u}_t}^{\perp}\hat{\v{v}}_t = \langle \nabla S_{t+1},\hat{\v{v}}^\perp_{t+1}\rangle 
\]
By definition of $\hat{\v{v}}_t$, we have
\begin{align*}
    \hat{\v{v}}_{t+1}^\perp & = P_{ \v{u}_{t+1}}^\perp\left[( \mat{I} - \eta \mat{H}_t)\hat{\v{v}}_t^\perp + \frac{\eta}{2}\nabla S_t^\perp(\delta_t^2 - \hat{x}_t^2)\right] - P_{ \v{u}_{t+1}}^\perp\tilde{\v{\xi}}_t\\
    & = P_{ \v{u}_{t+1}}^\perp( \mat{I} - \eta \mat{H}_t)\hat{\v{v}}_t^\perp + \frac{\eta}{2}P_{ \v{u}_{t+1}}^\perp\nabla S_t^\perp(\delta_t^2 - \hat{x}_t^2) - P_{ \v{u}_{t+1}}^\perp\tilde{\v{\xi}}_t
\end{align*}
We shall prove by induction that
\[
    \hat{\v{v}}_{t+1}^\perp  = P_{ \v{u}_{t+1}}^\perp\prod_{k=t}^0\mat{A}_k\hat{\v{v}}_0^\perp + \eta P_{ \v{u}_{t+1}}^\perp\sum_{s=0}^t\prod_{k=t}^{s+1}\mat{A}_k\left(\nabla S_s^{\perp}\frac{\delta_s^{2}-\hat{x}_s^{2}}{2} - P_{ \v{u}_{s+1}}^{\perp}\tilde{\v{\xi}}_s\right)
\]
where $\mat{A}_k:=(I-\eta \mat{H}_k)P_{ \v{u}_k}^{\perp}$. For base case, it is straight forward to hold as $\hat{\v{v}}_0^\perp = \hat{\v{v}}_0^\perp$. Assume that $\hat{\v{v}}_{t}^\perp$ takes the desired form. Then we have that
\begin{align*}
    \hat{\v{v}}_{t+1}^\perp & = P_{ \v{u}_{t+1}}^\perp( \mat{I} - \eta\mat{H}_t)\left(P_{ \v{u}_{t}}^\perp\prod_{k=t-1}^0\mat{A}_k\hat{\v{v}}_0^\perp + \eta P_{ \v{u}_{t}}^\perp\sum_{s=0}^{t-1}\prod_{k=t-1}^{s+1}\mat{A}_k\left(\nabla S_s^{\perp}\frac{\delta_s^{2}-\hat{x}_s^{2}}{2} - P_{ \v{u}_{s+1}}^{\perp}\tilde{\v{\xi}}_s\right)\right)\\
    & \quad\quad\quad + \eta P_{ \v{u}_{t+1}}^\perp\sum_{s=0}^t\prod_{k=t}^{s+1}\mat{A}_k\left(\nabla S_s^{\perp}\frac{\delta_s^{2}-\hat{x}_s^{2}}{2} - P_{ \v{u}_{s+1}}^{\perp}\tilde{\v{\xi}}_s\right)\\
    & = P_{ \v{u}_{t+1}}\prod_{k=t}^0\mat{A}_t\hat{\v{v}}_0^\perp + \eta P_{ \v{u}_{t+1}}^\perp\sum_{s=0}^t\prod_{k=t}^{s+1}\mat{A}_k\left(\nabla S_s^\perp \frac{\delta_s^{2}-\hat{x}_s^{2}}{2} - P_{ \v{u}_{s+1}}^{\perp}\tilde{\v{\xi}}_s\right)
\end{align*}
finishing the induction proof. This gives that
\begin{align*}
    \hat{y}_{t+1} & = \langle \nabla S_{t+1},\hat{\v{v}}^\perp_{t+1}\rangle \\
    & = \nabla S_{t+1}^\top P_{ \v{u}_{t+1}}^\perp\prod_{k=t}^0\mat{A}_k\hat{\v{v}}_0^\perp + \eta \nabla S_{t+1}^\top P_{ \v{u}_{t+1}}^\perp\sum_{s=0}^t\prod_{k=t}^{s+1}\mat{A}_k\nabla S_s^{\perp}\frac{\delta_s^{2}-\hat{x}_s^{2}}{2}\\
    & \quad\quad\quad - \eta \nabla S_{t+1}^\top P_{ \v{u}_{t+1}}^\perp\sum_{s=0}^t\prod_{k=t}^{s+1}\mat{A}_k P_{ \v{u}_{s+1}}^{\perp}\tilde{\v{\xi}}_s\\
    & = \mathcal{I}_{t+1} + \eta\sum_{s=0}^t\beta_{s\rightarrow t}\frac{\delta_s^2 -\hat{x}_s^2}{2} - \eta\sum_{s=0}^t\langle \v{\gamma}_{s\rightarrow t},\tilde{\v{\xi}}_s\rangle
\end{align*}
\end{proof}

\section{Heuristic Derivation of the Predicted Dynamics}\label{app:pred_dynamics}
 
The enhanced stochastic predicted dynamics (Definition~\ref{def:stoch_pred}) arise from a systematic Taylor expansion of the \sgd{} update rule around the constrained trajectory $\paramdagger_t$, keeping all leading-order terms and dropping subleading corrections.
We carry out this derivation in full.

\begin{lemma}\label{lem:heur_pred}
    Under Assumptions~\ref{ass:eigengap}-\ref{ass:scaling}, we have that
    \begin{align*}
        x_{t+1} & =  -(1 +\eta y_t)x_t - \frac{1}{2}\eta\kappa_tx_t^2 - \eta \langle\v{\xi}_t, \v{u}_{t+1}\rangle + \langle\v{u}_{t+1} - \v{u}_t, \v{v}_{t+1}\rangle  + O\left(E_t\right)\\
        P_{ \v{u}_{t+1}}^\perp\v{v}_{t+1} & = P_{ \v{u}_{t+1}}^\perp\left[( \mat{I} - \eta\mat{H})P_{ \v{u}_t}^\perp\v{v}_t + \eta \nabla S_t^\perp\frac{\delta_t^2 - x_t^2}{2}\right] - \eta P_{ \v{u}_{t+1}}^\perp\v{\xi}_t + O\left(E_t\right)
    \end{align*}
    where $E_t = \frac{\epsilon^2}{\delta}\norm{\v{v}_t}^2 + \frac{\epsilon^2}{\delta^2}\norm{\v{v}_t}^3 + \epsilon^3\delta$.
\end{lemma}

\begin{proof}
Recall that the \sgd{} update is $\param_{t+1} = \param_t - \eta\v{g}_{B_t}(\param_t)$.
Writing $\param_t = \paramdagger_t + \v{v}_t$ and decomposing $\v{g}_{B_t}(\param_t) = \nabla L(\param_t) + \v{\xi}_t$:
\begin{equation}\label{eq:derive_step1}
  \param_{t+1} = \param_t - \eta\nabla L(\param_t) - \eta\v{\xi}_t.
\end{equation}
The constrained trajectory satisfies (Lemma~\ref{lem:dagger_step}):
\[
\paramdagger_{t+1} = \paramdagger_t - \eta P_{ \v{u}_t,\nabla S_t}^{\perp}\nabla L_t + O(\epsilon^2\eta\norm{\nabla L_t}).
\]
Subtracting the two gives:
\begin{equation}\label{eq:derive_step1c}
  \v{v}_{t+1} = \v{v}_t - \eta\left[\nabla L(\paramdagger_t+\v{v}_t) - P_{ \v{u}_t,\nabla S_t}^{\perp}\nabla L_t\right] - \eta\v{\xi}_t  + O(\epsilon^2\eta\norm{\nabla L_t}).
\end{equation}
 
To expand the full-batch gradient, we invoke Lemma~\ref{lem:cubic_expansion}:
\[
  \nabla L(\paramdagger_t+\v{v}_t)
  = \nabla L_t +  \mat{H}_t\v{v}_t + \tfrac{1}{2}\nabla^{3} L_t(\v{v}_t,\v{v}_t) + O(\rho_4\norm{\v{v}_t}^{3}).
\]
Plugging into (\ref{eq:derive_step1c}) gives 
\begin{equation}\label{eq:vt_full}
    \v{v}_{t+1} = \v{v}_t - \eta\left[\nabla L_t - P_{ \v{u}_t,\nabla S_t}^{\perp}\nabla L_t\right] - \mat{H}_t\v{v}_t - \tfrac{1}{2}\nabla^{3} L_t(\v{v}_t,\v{v}_t) - \eta\v{\xi}_t + O(\eta\rho_4\norm{\v{v}_t}^{3} + \epsilon^2\eta\norm{\nabla L_t}).
\end{equation}
The projected gradient satisfies $\nabla L_t - P_{ \v{u}_t,\nabla S_t}^{\perp}\nabla L_t = P_{ \v{u}_t,\nabla S_t}\nabla L_t$.
Since $\ip{ \v{u}_t}{\nabla L_t}=0$ on $\calM$, this projection lies entirely in the $\nabla S_t^{\perp}$ direction. Therefore,
\begin{equation}\label{eq:derive_proj2}
  \nabla L_t - P_{ \v{u}_t,\nabla S_t}^{\perp}\nabla L_t = P_{ \v{u}_t,\nabla S_t}\nabla L_t
  = \frac{\ip{\nabla L_t}{\nabla S_t^{\perp}}}{\norm{\nabla S_t^{\perp}}^{2}}\nabla S_t^{\perp}
  = -\frac{\alpha_t}{\beta_t}\nabla S_t^{\perp}
  = -\frac{\delta_t^{2}}{2}\nabla S_t^{\perp},
\end{equation}
where we used $\alpha_t=-\ip{\nabla L_t}{\nabla S_t^{\perp}}$ (since $\ip{ \v{u}_t}{\nabla L_t}=0$), $\beta_t=\norm{\nabla S_t^{\perp}}^{2}$, and $\delta_t^{2}=2\alpha_t/\beta_t$.
 To handle the cubic term, we
decompose $\v{v}_t = x_t \v{u}_t + \v{v}_t^{\perp}$ where $\v{v}_t^{\perp}=P_{ \v{u}_t}^{\perp}\v{v}_t$.
The cubic term expands bilinearly (using symmetry of $\nabla^{3} L_t$):
\begin{equation}\label{eq:derive_cubic}
  \tfrac{1}{2}\nabla^{3} L_t(\v{v}_t,\v{v}_t)
  = \tfrac{1}{2}x_t^{2}\cdot\nabla^3L_t(\v{u}_t,\v{u}_t) + x_t\cdot \nabla^{3} L_t( \v{u}_t,\v{v}_t^{\perp}) + \tfrac{1}{2}\nabla^{3} L_t(\v{v}_t^{\perp},\v{v}_t^{\perp}).
\end{equation}
By Assumption~\ref{ass:sample_reg},\ref{ass:non_worst}, and noticing that $|x_t|,\norm{\v{v}_t^\perp}\leq \norm{\v{v}_t}$ we have that $\nabla^3L_t(\v{v}_t^\perp,\v{v}_t^\perp) \leq O(\epsilon\rho_3\norm{\v{v}_t}^2)$. Moreover, by Lemma~\ref{lem:sharpness_grad} additionally, we 
\[
\frac{1}{2}x_t^2\cdot \nabla^3L_t(\v{u}_t,\v{u}_t) = \frac{x_t^2}{2}\nabla S_t = \frac{x_t^2}{2}\nabla S_t^\perp + \frac{x_t^2}{2}\langle \nabla S_t, \v{u}_t\rangle\v{u}_t = \frac{x_t^2}{2}\nabla S_t^\perp + \frac{\kappa_t}{2}x_t^2\cdot \v{u}_t
\]
Combining gives
\begin{align*}
    \v{v}_{t+1} & = ( \mat{I} - \eta\mat{H}_t)\v{v}_t + \eta\cdot\frac{\delta_t^2-x_t^2}{2}\nabla S_t^\perp -  \frac{1}{2}\eta \kappa_tx_t^2\cdot \v{u}_t - \eta x_t\cdot \nabla^{3} L_t( \v{u}_t,\v{v}_t^{\perp}) - \eta\v{\xi}_t\\
    & \quad\quad\quad + O(\eta\epsilon\rho_3\norm{\v{v}_t}^2 + \rho_4\norm{\v{v}_t}^3 + \epsilon^2\eta\norm{\nabla L_t})
\end{align*}
Notice that, by Assumption~\ref{ass:prog_sharp} and the definition that $\beta_t = \norm{\nabla S_t}^2 = \Theta(\rho_3^2)$, we have 
\[
    \alpha_t = \Theta\left(\norm{\nabla L_t}\norm{\nabla S_t^\perp}\right) = \Theta(\rho_3\norm{\nabla L_t}) \Rightarrow \eta \norm{\nabla L_t} = \Theta\left(\frac{\eta\alpha_t}{\rho_3}\right) = \Theta\left(\eta \sqrt{\alpha_t} \cdot \frac{\sqrt{\alpha_t}}{\sqrt{\beta_t}}\right) = \Theta(\epsilon\delta)
\]
Also, we have that
\[
    \eta\rho_3\delta = \Theta\left(\eta \sqrt{\beta_t}\cdot \frac{\sqrt{\alpha_t}}{\sqrt{\beta_t}}\right) = \Theta (\eta\sqrt{\alpha_t}) = \Theta(\epsilon)
\]
Therefore, combining with Assumption~\ref{ass:scaling}, we have that
\begin{align*}
    \v{v}_{t+1} & = ( \mat{I} - \eta\mat{H}_t)\v{v}_t + \eta\cdot\frac{\delta_t^2- x_t^2}{2}\nabla S_t^\perp -  \frac{1}{2}\eta\kappa_tx_t^2\cdot \v{u}_t - \eta x_t\cdot \nabla^{3} L_t( \v{u}_t,\v{v}_t^{\perp}) - \eta\v{\xi}_t\\
    & \quad\quad\quad + O\left(\frac{\epsilon^2}{\delta}\norm{\v{v}_t}^2 + \frac{\epsilon^2}{\delta^2}\norm{\v{v}_t}^3 + \epsilon^3\delta\right)
\end{align*}
We will track the projection of $\v{v}_{t+1}$ onto $\v{u}_{t+1}$ and $P_{ \v{u}_{t+1}}^\perp$. To start, we have that
\begin{align*}
    x_{t+1} & = \langle \v{v}_{t+1},\v{u}_{t+1}\rangle\\
    & = \langle \v{v}_{t+1},\v{u}_{t}\rangle + \langle\v{u}_{t+1} - \v{u}_t, \v{v}_{t+1}\rangle\\
    & = \v{u}_t^\top( \mat{I} - \eta H_t)\v{v}_t - \frac{1}{2}\eta\kappa_tx_t^2 - \eta x_t\langle \nabla S_t, \v{v}_t^\perp \rangle - \eta \langle\v{\xi}_t, \v{u}_{t+1}\rangle\\
    & \quad\quad\quad + \langle\v{u}_{t+1} - \v{u}_t, \v{v}_{t+1}\rangle + O\left(\frac{\epsilon^2}{\delta}\norm{\v{v}_t}^2 + \frac{\epsilon^2}{\delta^2}\norm{\v{v}_t}^3 + \epsilon^3\delta\right)\\
    & = -(1 +\eta y_t)x_t - \frac{1}{2}\eta\kappa_tx_t^2 - \eta \langle\v{\xi}_t, \v{u}_{t+1}\rangle + \langle\v{u}_{t+1} - \v{u}_t, \v{v}_{t+1}\rangle\\
    & \quad\quad\quad + O\left(\frac{\epsilon^2}{\delta}\norm{\v{v}_t}^2 + \frac{\epsilon^2}{\delta^2}\norm{\v{v}_t}^3 + \epsilon^3\delta\right)
\end{align*}
Projecting onto $P_{ \v{u}_{t+1}}^\perp$ gives
\begin{align*}
    P_{ \v{u}_{t+1}}^\perp\v{v}_{t+1} & = P_{ \v{u}_{t+1}}^\perp\left[( \mat{I} - \eta\mat{H})\v{v}_t + \eta \nabla S_t^\perp\frac{\delta_t^2 - x_t^2}{2}\right] - \eta x_tP_{ \v{u}_{t+1}}^\perp\nabla^3L_t(\v{u}_t,\v{v}_t^\perp) - \eta P_{ \v{u}_{t+1}}^\perp\v{\xi}_t \\
    & \quad\quad\quad + O\left(\frac{\epsilon^2}{\delta}\norm{\v{v}_t}^2 + \frac{\epsilon^2}{\delta^2}\norm{\v{v}_t}^3 + \epsilon^3\delta\right)
\end{align*}
We have that
\[
    \eta x_tP_{ \v{u}_{t+1}}^\perp\nabla^3L_t(\v{u}_t,\v{v}_t^\perp) = O\left(\eta |x_t|\norm{P_{ \v{u}_{t}}^\perp\nabla^3L_t(\v{u}_t,\v{v}_t^\perp)}\right) + O\left(\eta|x_t|\norm{P_{ \v{u}_{t+1}} - P_{ \v{u}_t}}\norm{\nabla^3L_t(\v{u}_t,\v{v}_t^\perp)}\right)
\]
where, by Assumption~\ref{ass:non_worst},
\[                 
    \norm{P_{ \v{u}_{t}}^\perp\nabla^3L_t(\v{u}_t,\v{v}_t^\perp)} \leq O(\epsilon \rho_3\norm{\v{v}_t})
\]
and by Lemma~\ref{lem:lipschitz}
\[
    \norm{P_{ \v{u}_{t+1}} - P_{ \v{u}_t}} \leq O(\norm{\v{u}_{t+1} - \v{u}_t}) \leq O(\eta^2\rho_3\norm{\nabla L_t}) \leq O(\epsilon^2)
\]
Thus, we have that
\[
    \norm{\eta x_tP_{ \v{u}_{t+1}}^\perp\nabla^3L_t(\v{u}_t,\v{v}_t^\perp)} \leq O(\eta\epsilon\rho_3\norm{\v{v}_t}^2 + \eta\epsilon^2\rho_3\norm{\v{v}_t}^2) \leq O\left(\epsilon^2\frac{\norm{\v{v}_t}^2}{\delta}\right)
\]
Therefore, the projection of $\v{v}_{t+1}$ onto $P_{ \v{u}_{t+1}}$ is given by
\begin{align*}
    P_{ \v{u}_{t+1}}^\perp\v{v}_{t+1} & = P_{ \v{u}_{t+1}}^\perp\left[( \mat{I} - \eta\mat{H})\v{v}_t + \eta \nabla S_t^\perp\frac{\delta_t^2 - x_t^2}{2}\right] - \eta P_{ \v{u}_{t+1}}^\perp\v{\xi}_t + O\left(\frac{\epsilon^2}{\delta}\norm{\v{v}_t}^2 + \frac{\epsilon^2}{\delta^2}\norm{\v{v}_t}^3 + \epsilon^3\delta\right)\\
    & = P_{ \v{u}_{t+1}}^\perp\left[( \mat{I} - \eta\mat{H})P_{ \v{u}_t}^\perp\v{v}_t + \eta \nabla S_t^\perp\frac{\delta_t^2 - x_t^2}{2}\right] - \eta P_{ \v{u}_{t+1}}^\perp\v{\xi}_t + O\left(\frac{\epsilon^2}{\delta}\norm{\v{v}_t}^2 + \frac{\epsilon^2}{\delta^2}\norm{\v{v}_t}^3 + \epsilon^3\delta\right)
\end{align*}
where the last equality follows from Assumption~\ref{ass:eigengap}.
\end{proof}
 
% ════════════════════════════════════════════════════════════
\section{Complete Proof of the Stochastic Coupling Theorem}\label{app:coupling_proof}
% ════════════════════════════════════════════════════════════
 
The proof follows a four-lemma structure: $(i)$~local growth of the enhanced step map; $(ii)$~one-step local truncation error; $(iii)$~mean-square boundedness; $(iv)$~mean-square coupling.
 
\subsection{Local Growth of the Enhanced Step Map}
 
\begin{lemma}[Local Growth Bound]\label{lem:local_growth}
Let Assumption~\ref{ass:eigengap}-\ref{ass:scaling} hold. If $\norm{\v{v}}\le R\delta$ for some $R$ satisfying $1 \le R \le O(\epsilon^{-1})$, then
\begin{equation}\label{eq:local_growth}
\norm{\estep_t(\v{v})}^{2}\le(1 + O(\epsilon R))\norm{\v{v}}^2 + O(\epsilon R^3\delta^2).
\end{equation}
\end{lemma}
 
\begin{proof}
Let $x=\ip{ \v{u}_t}{\v{v}}$, $y=\ip{\nabla S_t^{\perp}}{\v{v}}$. To start, we bound the $ \v{u}_{t+1}$-component. By definition of the enhanced step map:
\[
  \langle \v{u}_{t+1},\estep_t(\v{v})\rangle  = -(1+\eta y)x - \tfrac{\eta}{2}\kappa_t x^{2}.
\]
We bound each factor.
Since $\eta|y| \le \eta\norm{\nabla S_t^{\perp}}\norm{\v{v}} \le \eta\rho_3 R\delta$ and $\eta\rho_3\delta = O(\epsilon)$, we have $\eta|y| = O(\epsilon R)$.
Similarly, $\frac{\eta}{2}|\kappa_t||x| \le \frac{\eta}{2}\rho_3 R\delta = O(\epsilon R)$. Therefore,
\begin{equation}\label{eq:wu_growth}
  |\langle \v{u}_{t+1},\estep_t(\v{v})\rangle| \le (1 + O(\epsilon R))|x| 
\end{equation}
Squaring: $\langle \v{u}_{t+1},\estep_t(\v{v})\rangle^{2} \le (1+O(\epsilon R))^{2}|x|^{2} \le (1+O(\epsilon R))|x|^{2}$ since $R \leq O(\epsilon^{-1})$. Next, we bound the $P_{ \v{u}_{t+1}}^{\perp}$-component.
\begin{align*}
  P_{ \v{u}_{t+1}}^\perp \estep_t(\v{v}) = 
  &= P_{ \v{u}_{t+1}}^{\perp}\left[(I-\eta \mat{H}_t)P_{ \v{u}_t}^{\perp}\v{v} + \eta\nabla S_t^{\perp}\tfrac{\delta_t^{2}-x^{2}}{2}\right].
\end{align*}
Notice that for all $\v{v}^\perp \perp \v{u}_t$, we have that
\[
    \norm{(I-\eta \mat{H}_t)\v{v}^\perp} \leq \max_{i \geq 2}|1 - \lambda_i(\mat{H}_t)|\norm{\v{v}^\perp}
\]
By Assumption~\ref{ass:eigengap} and Assumption~\ref{ass:non_worst}, we obtain that
\[
    \max_{i \geq 2}|1 - \lambda_i(\mat{H}_t)| \leq 1 + O(\epsilon) \Rightarrow \norm{(I-\eta \mat{H}_t)P_{ \v{u}_t}\v{v}} \leq (1 + O(\epsilon))\norm{\v{v}}
\]
To bound the term $\eta\nabla S_t^{\perp}\tfrac{\delta_t^{2}-x^{2}}{2}$, we notice that $\norm{\nabla S_t^\perp} \leq \rho_3$ and $|\delta_t^2 - x_t^2| \leq (1 + R^2)\delta^2$. Therefore,
\[
    \norm{\eta\nabla S_t^{\perp}\tfrac{\delta_t^{2}-x^{2}}{2}} \leq O((1+R^2)\eta\rho_3\delta^2) \leq O((1+R^2)\epsilon\delta)
\]
Therefore,
\[
    \norm{P_{ \v{u}_{t+1}}^\perp \estep_t(\v{v})}^2 \leq (1 + O(\epsilon R))\norm{\v{v}}^2 + O(\epsilon R^3\delta^2)
\]
which implies that
\[
    \norm{\estep_t(\v{v})}^2 \leq (1 + O(\epsilon R))\norm{\v{v}}^2 + O(\epsilon R^3\delta^2)
\]
\end{proof}
 
\subsection{One-Step Local Truncation Error}
 
\begin{lemma}[One-Step Local Truncation Error]\label{lem:local_error}
Let Assumption~\ref{ass:eigengap}-\ref{ass:scaling} hold. Define $\v{\varepsilon}_t = \v{v}_{t+1} - \estep_t^S(\v{v}_t)$. If $\norm{\v{v}_t}\leq R\delta$ for some $R$ satisfying $1 \leq R\leq O(\epsilon^{-1})$, then we have that
\begin{align*}
    \E{\norm{\v{\varepsilon}_t}^2\mid \calF_t} \leq O(\epsilon^4\eta^2)\E{\norm{\v{\xi}_t}^2\mid \calF_t} + O\left(\frac{\epsilon^4}{\delta^2}\norm{\v{v}_t}^4 + \frac{\epsilon^4}{\delta^6}\norm{\v{v}_t}^6 + \epsilon^4\norm{\v{v}_t}^2 + \epsilon^5 R^3\delta^2\right)
\end{align*}
\end{lemma}
 
\begin{proof}
Recall that, by definition, $\estep_t^S(\v{v}_t) = \estep_t(\v{v}_t) - \eta\v{\xi}_ts$ with 
\begin{align*}
    P_{ \v{u}_{t+1}}^\perp\estep_t(\v{v}_t) &= P_{ \v{u}_{t+1}}^\perp\left[\mat{A}_t\v{v}_t + \eta\nabla S_t^\perp\frac{\delta_t^2 - x_t^2}{2}\right]\\
    \langle \v{u}_{t+1},\estep_t(\v{v}_t)\rangle & = - (1+\eta y_t)x_t - \frac{\eta}{2}\kappa_tx_t^2
\end{align*}
Thus, let $\zeta_t = \langle \v{u}_{t+1}, \v{\xi}_t\rangle$, we have
\begin{align*}
    P_{ \v{u}_{t+1}}^\perp\estep_t^S(\v{v}_t) &= P_{ \v{u}_{t+1}}^\perp\left[\mat{A}_t\v{v}_t + \eta\nabla S_t^\perp\frac{\delta_t^2 - x_t^2}{2}\right] - \eta P_{ \v{u}_{t+1}}^\perp\v{\xi}\\
    \langle \v{u}_{t+1},\estep_t^S(\v{v}_t)\rangle & = - (1+\eta y_t)x_t - \frac{\eta}{2}\kappa_tx_t^2 - \eta\zeta_t
\end{align*}
By Lemma~\ref{lem:heur_pred}, we have that
\begin{align*}
    P_{ \v{u}_{t+1}}^\perp\v{v}_{t+1} & = P_{ \v{u}_{t+1}}^\perp\left[\mat{A}_t\v{v}_t + \eta \nabla S_t^\perp\frac{\delta_t^2 - x_t^2}{2}\right] - \eta P_{ \v{u}_{t+1}}^\perp\v{\xi}_t + O\left(E_t\right)\\
    \langle \v{u}_{t+1},\v{v}_{t+1}\rangle & =  -(1 +\eta y_t)x_t - \frac{1}{2}\eta\kappa_tx_t^2 - \eta \zeta_t + \langle\v{u}_{t+1} - \v{u}_t, \v{v}_{t+1}\rangle  + O\left(E_t\right)
\end{align*}
where $E_t = \frac{\epsilon^2}{\delta}\norm{\v{v}_t}^2 + \frac{\epsilon^2}{\delta^2}\norm{\v{v}_t}^3 + \epsilon^3\delta$. Therefore $\norm{P_{ \v{u}_{t+1}}^\perp\v{v}_{t+1} - P_{ \v{u}_{t+1}}^\perp\estep_t^S(\v{v}_t)} = O(E_t)$ and
\begin{align*}
    \langle \v{u}_{t+1}, \v{v}_{t+1} - \estep_t^S(\v{v}_t)\rangle = \langle\v{u}_{t+1} - \v{u}_t, \v{v}_{t+1}\rangle  + O\left(E_t\right)
\end{align*}
This gives that
\begin{align*}
    \v{\varepsilon}_t & = \langle\v{u}_{t+1} - \v{u}_t, \v{v}_{t+1}\rangle\v{u}_{t+1} + O\left(\frac{\epsilon^2}{\delta}\norm{\v{v}_t}^2 + \frac{\epsilon^2}{\delta^2}\norm{\v{v}_t}^3 + \epsilon^3\delta\right)
\end{align*}
which implies that
\begin{align*}
    \norm{\v{\varepsilon}_t}^2 & \leq 2\norm{\v{u}_{t+1} - \v{u}_t}^2\norm{\v{v}_{t+1}}^2 + O\left(\frac{\epsilon^4}{\delta^2}\norm{\v{v}_t}^4 + \frac{\epsilon^4}{\delta^4}\norm{\v{v}_t}^6 + \epsilon^6\delta^2\right)\\
    & \leq O(\epsilon^4)\norm{\v{v}_{t+1}}^2 + O\left(\frac{\epsilon^4}{\delta^2}\norm{\v{v}_t}^4 + \frac{\epsilon^4}{\delta^4}\norm{\v{v}_t}^6 + \epsilon^6\delta^2\right)
\end{align*}
Next we bound $\norm{\v{v}_{t+1}}$. We notice that
\[
    \v{v}_{t+1} = \estep_t^S(\v{v}_t) + \v{\varepsilon}_t = \estep_t^S(\v{v}_t) + \v{\varepsilon}_t = \estep_t(\v{v}_t) - \eta\v{\xi}_t + \v{\varepsilon}_t
\]
Therefore,
\begin{align*}
    \E{\norm{\v{v}_{t+1}}^2\mid \calF_t} & = 2\E{\norm{\estep_t(\v{v}_t)}^2\mid \calF_t} + 2\E{\norm{\v{\varepsilon}_t}^2\mid \calF_t} + \eta^2\E{\norm{\v{\xi}_t}^2\mid \calF_t}\\
    & \leq O(1)\norm{\v{v}_t}^2 + O(\epsilon R^3\delta^2) + 2\E{\norm{\v{\varepsilon}_t}^2\mid \calF_t} + \eta^2\E{\norm{\v{\xi}_t}^2\mid \calF_t}
\end{align*}
Plugging back into the bound on $\v{\varepsilon}_t$ gives
\begin{align*}
    \E{\norm{\v{\varepsilon}_t}^2\mid \calF_t} & \leq O(\epsilon^4)\E{\norm{\v{\varepsilon}_t}^2\mid \calF_t} + O(\epsilon^4\eta^2)\E{\norm{\v{\xi}_t}^2\mid\calF_t} \\
    & \quad\quad\quad O\left(\frac{\epsilon^4}{\delta^2}\norm{\v{v}_t}^4 + \frac{\epsilon^4}{\delta^6}\norm{\v{v}_t}^6 + \epsilon^4 \norm{\v{v}_t}^2 + \epsilon^5 R^3\delta^2\right)
\end{align*}
With the assumption that $\epsilon \leq o(1)$ gives
\[
    \E{\norm{\v{\varepsilon}_t}^2\mid \calF_t} \leq O(\epsilon^4\eta^2)\E{\norm{\v{\xi}_t}^2\mid\calF_t} + O\left(\frac{\epsilon^4}{\delta^2}\norm{\v{v}_t}^4 + \frac{\epsilon^4}{\delta^6}\norm{\v{v}_t}^6 + \epsilon^4\norm{\v{v}_t}^2 + \epsilon^5 R^3\delta^2\right)
\]
\end{proof}
\subsection{Mean-Square Boundedness}
 
\begin{lemma}[Mean-Square Boundedness of the Enhanced Predicted Dynamics]\label{lem:bound_pred_full}
Let Assumption~\ref{ass:eigengap}-\ref{ass:scaling} hold. If $\norm{\v{v}_t}\leq R\delta$ for some $R$ satisfying $1 \leq R \leq O(\epsilon^{-1})$, then
\begin{equation}
  \E{\norm{\hat{\v{v}}_{t+1}}^{2}\mid\calF_t}
  \le (1+O(\epsilon R))\norm{\hat{\v{v}}_t}^{2}+O(\epsilon R^3\delta^2).
\end{equation}
\end{lemma}
 
\begin{proof}
By definition, $\hat{\v{v}}_{t+1}=\estep_t^{S}(\hat{\v{v}}_t)=\estep_t(\hat{\v{v}}_t)-\eta\v{\xi}_t$.
Since $\E{\v{\xi}_t\mid\calF_t}=\v{0}$, the cross-term vanishes when squaring:
\begin{equation}\label{eq:pred_sq_expand}
  \E{\norm{\hat{\v{v}}_{t+1}}^{2}\mid\calF_t}
  = \E{\norm{\estep_t(\hat{\v{v}}_t)}^{2}}+\eta^{2}\E{\norm{\v{\xi}_t}^{2}}.
\end{equation}
Taking full expectations and applying Lemma~\ref{lem:local_growth}:
\[
  \E{\norm{\hat{\v{v}}_{t+1}}^{2}\mid\calF_t}
  \le (1+O(\epsilon R))\norm{\hat{\v{v}}_t}^{2} + O(\epsilon R^3\delta^2 + \eta^2\sigma^2).
\]
where we also applied that $\E{\norm{\v{\xi}_t}^2}\leq \sigma^2$.
By the scaling condition in Assumption~\ref{ass:scaling}, $\eta^{2}\sigma^{2}\le O(\epsilon\delta^{2})$. Therefore:
\begin{equation}\label{eq:gronwall_recursion}
  \E{\norm{\hat{\v{v}}_{t+1}}^{2}\mid\calF_t}
  \le (1+O(\epsilon R))\norm{\hat{\v{v}}_t}^{2}+O(\epsilon R^3\delta^2).
\end{equation}
\end{proof}

\subsection{Lipschitz Bound of Enhanced Step Map}

\begin{lemma}[Lipschitz Bound of Enhanced Step Map]\label{lem:lip_step}
    Let Assumption~\ref{ass:eigengap}-\ref{ass:scaling} hold. For $\v{v},\v{w}$ satisfying $\norm{\v{v}},\norm{\v{w}}\leq R\delta$ for some $1 \leq R \leq O(\epsilon^{-1})$, we have that
    \[
        \norm{\estep_t^S(\v{v}) - \estep_t^S(\v{w})} \leq (1 + O(\epsilon R))\norm{\v{v} - \v{w}}
    \]
\end{lemma}
\begin{proof}
    By definition, we have that
    \begin{align*}
        \estep_t^S(\v{v}) - \estep_t^S(\v{w}) & = \estep_t(\v{v}) - \estep_t(\v{w})\\
        & = P_{ \v{u}_{t+1}}^\perp\mat{A}_t(\v{v} - \v{w}) + \frac{\eta}{2} P_{ \v{u}_{t+1}}^\perp\nabla S_t^\perp (x_v^2 - x_w^2)\\
        & \quad\quad\quad - \left((1 + \eta y_v)x_v - (1+\eta y_w)x_w\right) \v{u}_{t+1} - \frac{\eta}{2}\kappa_t(x_v^2 - x_w^2) \v{u}_{t+1}\\
        & = \mat{A}_t(\v{v} - \v{w}) + \frac{\eta}{2}\nabla S_t^\perp (x_v^2 - x_w^2)  - \eta\left(y_v x_v - y_w  x_w\right) \v{u}_{t+1}
    \end{align*}
    where $x_v = \langle  \v{u}_{t}, \v{v}\rangle, x_w = \langle  \v{u}_t, \v{w}\rangle$ and $y_v = \langle \nabla S_t^\perp, \v{v}\rangle,y_w = \langle \nabla S_t^\perp, \v{w}\rangle$. Therefore
    \begin{align*}
        \norm{\estep_t^S(\v{v}) - \estep_t^S(\v{w})} & \leq \underbrace{\norm{\mat{A}_t}\norm{\v{v} - \v{w}}}_{T_1} + \underbrace{\frac{\eta}{2}\norm{\nabla S_t}\left|x_v^2 - x_w^2\right|}_{T_2} + \underbrace{\left|y_v x_v - y_w  x_w\right|}_{T_3}.
    \end{align*}
    Next, we will analyze each term individually. To start, for $T_1$ we have that
    \[
        \norm{\mat{A}_t} \leq \norm{ \mat{I} - \eta\mat{H}_t} \leq 1 + \eta\lambda_{\min}(\mat{H}_t) \leq 1 + O(\eta\epsilon)\norm{\mat{H}_t} \leq 1 + O(\epsilon)
    \]
    Therefore, $T_1 \leq (1 + O(\epsilon))\norm{\v{v} - \v{w}}$. For $T_2$, we have that $\norm{\nabla S_t} \leq O(\rho_3)$ and
    \[
        \left|x_v^2 - x_w^2\right| \leq |x_v + x_w|\cdot |x_v - x_w| \leq (\norm{\v{v}} + \norm{\v{w}})\norm{\v{v} - \v{w}} \leq O(R\delta)\norm{\v{v} - \v{w}}
    \]
    Therefore, $T_2 \leq O(\eta\rho_3\delta R)\norm{\v{v} - \v{w}} \leq O(\epsilon R)\norm{\v{v} - \v{w}}$. Lastly, for $T_3$ we have that
    \begin{align*}
        T_3 & \leq \eta \left|y_v\right|\left|x_v - x_w\right| + \eta |x_w| |y_v - y_w|\\
        & \leq O(\eta\rho_3\delta R)\norm{\v{v} - \v{w}} + \eta\delta\cdot O(\rho_3)\norm{\v{v} - \v{w}}\\
        & \leq O(\eta\rho_3\delta R)\norm{\v{v} - \v{w}}\\
        & \leq O(\epsilon R)\norm{\v{v} - \v{w}}
    \end{align*}
    where in the second inequality we used the definition $y_v = \langle \nabla S_t^\perp, \v{v}\rangle,y_w = \langle \nabla S_t^\perp, \v{w}\rangle$ and the fact that $\norm{\nabla S_t^\perp} \leq O(\rho_3)$. Combining all three terms gives the desired result.
\end{proof}
 
\subsection{Mean-Square Coupling}
 
\begin{lemma}[Mean-Square Coupling]\label{lem:coupling_full}
Let Assumption~\ref{ass:eigengap}-\ref{ass:scaling} hold with a small enough constant $c_{\mathscr{T}}$. If $\norm{\v{v}_0}^2 \leq c_0\delta^2$ for some small enough constant $c_0$, then for any small constant $\omega$ such that $\omega^{-1}\leq O(1)$, we have that with proability at least $1 - \omega$, the following holds for all $t\leq \mathscr{T}$
\[
    \norm{\v{v}_t - \hat{\v{v}}_t} \leq O(\epsilon\delta);\quad \norm{\v{v}_t} \leq O(\delta);\quad \norm{\hat{\v{v}}_t} \leq O(\delta)
\]
\end{lemma}
 
\begin{proof}
Define $\v{e}_t:=\v{v}_t-\hat{\v{v}}_t$. Recall that in Lemma~\ref{lem:local_error}, we defined $\v{\varepsilon}_t =\v{v}_{t+1} - \estep_t^S(\v{v}_t)$. Then we have that
\[
    \v{e}_{t+1} =\v{v}_{t+1}-\hat{\v{v}}_{t+1} = \estep_t^S(\v{v}_t) - \estep)_t^S(\hat{\v{v}}_t) + \v{\varepsilon}_t
\]
This gives that
\[
    \norm{\v{e}_{t+1}} \leq \norm{\estep_t^S(\v{v}_t) - \estep_t^S(\hat{\v{v}}_t)} + \norm{\v{\varepsilon}_t}
\]
By Lemma~\ref{lem:lip_step}, the Lipschitzness of $\estep_t^S(\cdot)$ can be bounded as
\[
    \norm{\estep_t^S(\v{v}_t) - \estep_t^S(\hat{\v{v}}_t)} \leq (1 + O(\epsilon R))\norm{\v{v}_t - \hat{\v{v}}_t} = (1 + O(\epsilon R))\norm{\v{e}_t}
\]
Therefore, we have that
\[
    \E{\norm{\v{e}_{t+1}}^2\mid \calF_t} \leq (1+ O(\epsilon R))\norm{\v{e}_t}^2 + O(\epsilon^{-1}R^{-1})\E{\norm{\v{\varepsilon}_t}^2\mid \calF_t}
\]
Define $\tau:=\inf\{t\le\mathscr{T}:\norm{\v{v}_t}>M\delta\text{ or }\norm{\hat{\v{v}}_t}>M\delta\}$ for a constant $M \geq 1$ to be chosen. Then for all $t < \tau$, we have that 
\begin{align*}
    O(\epsilon^{-1}R^{-1})\E{\norm{\v{\varepsilon}_t}^2\mid \calF_t} & \leq O(\epsilon^{-1}M^{-1})\cdot O(\sigma^2\eta^2\epsilon^4 + \epsilon^4\delta^2M^4 + \epsilon^4M^6 + \epsilon^4\delta^2M^2 + \epsilon^5M^3\delta^2)\\
    & \leq O(M^3\epsilon^3(\delta^2 + M^2))
\end{align*}
This gives that when $t < \tau$
\[
    \E{\norm{\v{e}_{t+1}}^2\mid \calF_t} \leq (1+ O(M\epsilon))\norm{\v{e}_t}^2 +  O(M^3\epsilon^3(\delta^2 + M^2))
\]
In the meantime, by Lemma~\ref{lem:bound_pred_full} we have that
\[
    \E{\norm{\hat{\v{v}}_{t+1}}^2\mid \calF_t} \leq  (1+ O(M\epsilon))\norm{\hat{\v{v}}_t}^2 + O(M^3\epsilon\delta^2)+ O(M^3\epsilon\delta^2)
\]
Therefore, there exists constant $C_1, C_2, C_3$ such that
\begin{gather*}
    \E{\norm{\v{e}_{t+1}}^2\mid \calF_t} \leq (1+ C_1M\epsilon)\norm{\v{e}_t}^2 +  C_1C_2M^3\epsilon^3(\delta^2 + M^2)\\
    \E{\norm{\hat{\v{v}}_{t+1}}^2\mid \calF_t} \leq  \norm{\hat{\v{v}}_t}^2 + C_3M^3\epsilon\delta^2
\end{gather*}
Define $F_1, G_1$ as
\begin{gather*}
    F_t^{(1)} = \norm{\v{e}_t}^2 + C_2M^2\epsilon^2(\delta^2 + M^2);\quad 
    G_t^{(1)} = (1 + C_1 M\epsilon)^{-t}F_t^{(1)}
\end{gather*}
Then it holds that
\begin{align*}
    \E{F_{t+1}^{(1)}\mid \calF_t} & \leq (1+ C_1M\epsilon)\norm{\v{e}_t}^2 +  C_1C_2M^3\epsilon^3(\delta^2 + M^2) + C_2M^2\epsilon^2(\delta^2 + M^2)\\
    & \leq (1+ C_1M\epsilon)\left(\norm{\v{e}_t}^2 + C_2M^2\epsilon^2(\delta^2 + M^2)\right).
\end{align*}
Thus, $\E{G_{t+1}^{(1)}\mid \calF_t} \leq G_t^{(1)}$, which implies that $G_t^{(1)}$ is a supermartingale. Similarly, we define
\[
    G_t^{(2)} = \norm{\hat{\v{v}}_t}^2 - C_3M^3\epsilon\delta^2t.
\]
Then we have that
\begin{align*}
    \E{G_{t+1}^{(2)}} & = \E{\norm{\hat{\v{v}}_{t+1}}^2} - C_3M^3\epsilon\delta^2(t+1)\\
    & = \norm{\hat{\v{v}}_t}^2 + C_3M^3\epsilon\delta -C_3M^3\epsilon\delta^2(t+1)\\
    & = \norm{\hat{\v{v}}_t}^2 -C_3M^3\epsilon\delta^2t\\
    & = G_t^{(2)}
\end{align*}
which shows that $G_t^{(2)}$ is also a supermartingale.
Applying Doob's maximal inequality gives that
\[
    \Pr{\max_{t\leq \mathscr{T}}G_t^{(1)} \geq B_1}\leq B_1^{-1}\E{G_0^{(1)}};\quad \Pr{\max_{t\leq \mathscr{T}}G_t^{(2)} \geq B_2}\leq B_2^{-1}\E{G_0^{(2)}}
\]
By definition,
\begin{gather*}
    \E{G_0^{(1)}} = \E{F_0^{(1)}} = \E{\norm{\v{e}_0}^2} + C_2M^2\epsilon^2(\delta^2+M^2) = C_2M^2\epsilon^2(\delta^2+M^2)\\
    \E{G_0^{(1)}}  = \E{\norm{\hat{\v{v}}_0}^2} = \norm{\v{v}_0}^2
\end{gather*}
Let $\omega < \frac{1}{4}$ be a small constant. Choose $B_1 = \omega^{-1}C_2M^2\epsilon^2(\delta^2+M^2)$ and $B_2 = \omega^{-1}\norm{\v{v}_0}^2$. Then with probability at least $1 - 2\omega$ we have that
\[
    G_t^{(1)} \leq \omega^{-1}C_2M^2\epsilon^2(\delta^2+M^2);\quad G_t^{(2)} \leq \omega^{-1}\norm{\v{v}_0}^2
\]
which, under the assumption that $t\leq \mathscr{T} \leq \sfrac{c_{\mathscr{T}}}{\epsilon}$, implies that
\begin{align*}
    \norm{\v{e}_t}^2 & \leq (1 + C_1M\epsilon)^t\omega^{-1}C_2M^2\epsilon^2(\delta^2+M^2)\\
    & \leq \omega^{-1}\exp(C_1M\epsilon t)C_2M^2\epsilon^2(\delta^2+M^2)\\
    & \leq \omega^{-1}\exp(c_{\mathscr{T}}C_1M)C_2M^2\epsilon^2(\delta^2+M^2)\\
    & \leq C_4^2M^2\epsilon^2(\delta^2 + M^2)
\end{align*}
for some constant $C_4$. Also, under the assumption that $\norm{\v{v}_0}^2 \leq c_0\delta^2$,
\[
    \norm{\hat{\v{v}}_t}^2 \leq \omega^{-1}\norm{\v{v}_0}^2 + C_3M\epsilon\delta^2t \leq \omega^{-1}\norm{\v{v}_0}^2 + c_{\mathscr{T}}C_3M^3\delta^2 \leq (\sfrac{c_0}{\omega} + c_{\mathscr{T}}C_3M^3)\delta^2 \leq \delta^2
\]
for small enough $c_{\mathscr{T}}$ and $c_0$
Altogether, this implies that 
\begin{align*}
    \norm{\v{v}_t} \leq \norm{\hat{\v{v}}_t} + \norm{\v{e}_t} \leq \delta + C_4M\epsilon(\delta + M) \leq 2\delta
\end{align*}
by taking $M = 3$ and $\epsilon \leq o(1)$. This gives that for all $t\leq \mathscr{T}$, $t \leq \tau$. Thus, for all $t\leq \mathscr{T}$, we have that
\begin{gather*}
    \norm{\v{v}_t - \hat{\v{v}}_t} \leq O(\epsilon\delta);\quad \norm{\v{v}_t} \leq O(\delta);\quad \norm{\hat{\v{v}}_t} \leq O(\delta)
\end{gather*}
\end{proof}

\subsection{Sharpness Approximation}
\begin{lemma}[Sharpness Approximation]\label{lem:sharp_approx}
Under Assumptions~\ref{ass:eigengap}-\ref{ass:scaling}, if $\norm{\v{v}_t - \hat{\v{v}}_t} \leq O(\epsilon\delta)$ and $\norm{\hat{\v{v}}_t} \leq O(\delta)$, then for any $\param_t = \paramdagger_t + \v{v}_t$ we have
\begin{equation}\label{eq:sharp_approx_full}
   S(\param_t)
  = \frac{2}{\eta} + \hat{y}_t + \kappa_t \hat{x}_t + O\left(\frac{\epsilon^2}{\eta}\right)
\end{equation}
\end{lemma}
 
\begin{proof}
Recall that $\param_t = \param_t^\dagger + \v{v}_t$. Applying Taylor expansion to $S(\param_t)$ gives that
\[
    S(\param_t) = S(\param_t^\dagger) + \langle\nabla S(\param_t^\dagger),\v{v}_t\rangle + \int_0^1\lambda\v{v}_t^\top\nabla^2S(\param_t^\dagger + \lambda\v{v}_t)\v{v}_td\lambda
\]
Where the last term can be bounded as
\[
    \left|\int_0^1\lambda\v{v}_t^\top\nabla^2f(\param_t^\dagger + \lambda\v{v}_t)\v{v}_td\lambda\right| \leq \norm{\nabla^2S(\param_t^\dagger + \lambda \v{v}_t)}\norm{\v{v}_t}^2 \leq O(\rho_4)\norm{\v{v}_t}^2
\]
Since $\norm{\v{v}_t} \leq O(\delta)$ and $\rho_4 \leq O(\eta\rho_3^2)$ by Assumption~\ref{ass:scaling}, we have that
\[
    \left|\int_0^1\lambda\v{v}_t^\top\nabla^2f(\param_t^\dagger + \lambda\v{v}_t)\v{v}_td\lambda\right|  \leq O(\eta\rho_3^2\delta^2)\leq O(\eta^{-1}\epsilon^2)
\]
Next, we decompose $\langle\nabla S(\param_t^\dagger),\v{v}_t\rangle = \langle\nabla S_t,\v{v}_t\rangle$. We write
\[
    \nabla S_t = \langle \nabla S_t,  \v{u}_t\rangle  \v{u}_t + \nabla S_t^\perp = \kappa_t  \v{u}_t + \nabla S_t^\perp
\]
Therefore
\begin{align*}
    \langle\nabla S_t,\v{v}_t\rangle & = \langle\nabla S_t,\hat{\v{v}}_t\rangle + \langle\nabla S_t,\v{v}_t - \hat{\v{v}}_t\rangle\\
    & = \kappa_t\langle  \v{u}_t,\hat{\v{v}}_t\rangle + \langle \nabla S_t^\perp,\hat{\v{v}}_t\rangle + O(\norm{\nabla S_t}\norm{\v{v}_t - \hat{\v{v}}_t})\\
    & = \kappa_t\hat{x}_t + \hat{y}_t + O(\rho_3\epsilon\delta)\\
    & = \kappa_t\hat{x}_t + \hat{y}_t + O(\eta^{-1}\epsilon^2)
\end{align*}
where we used $\norm{\v{v} - \hat{\v{v}}_t} \leq O(\epsilon\delta)$. Noticing that $S(\param_t)=\frac{2}{\eta}$ by Lemma~\ref{lem:dagger_step}, and combining the results gives that
\[
    S(\param_t) = \frac{2}{\eta} + \kappa_t\hat{x}_t + \hat{y}_t + O(\eta^{-1}\epsilon^2)
\]
\end{proof}
 
\subsection{Proof of Theorem~\ref{thm:coupling}}\label{app:coupling_main}
 
\begin{proof}[Proof of Theorem~\ref{thm:coupling}]
Under the assumptions of the theorem, Lemma~\ref{lem:coupling_full} implies that $t\leq \mathscr{T}$
\[
    \norm{\v{v}_t - \hat{\v{v}}_t} \leq O(\epsilon\delta);\quad \norm{\v{v}_t} \leq O(\delta);\quad \norm{\hat{\v{v}}_t} \leq O(\delta)
\]
with probability at least $1 - \omega$. Assume that such event happens. Then by Lemma~\ref{lem:sharp_approx} we have that
\[
     S(\param_t)
  = \frac{2}{\eta} + \hat{y}_t + \kappa_t \hat{x}_t + O\left(\frac{\epsilon^2}{\eta}\right)
\]
It suffice to show the loss approximation. Perform Taylor expansion for the loss gives
\[
    L(\param_t) = L(\param_t^\dagger) + \langle\nabla L(\param_t^\dagger),\v{v}_t\rangle + \frac{1}{2}\v{v}^\top\nabla^2L(\param_t^\dagger)\v{v}_t + O(\rho_3\norm{\v{v}_t}^2)
\]
Recall that $\alpha_t = -\langle \nabla L_t, \nabla S_t\rangle = \Theta(\norm{\nabla L_t}\norm{\nabla S_t^\perp})$ by Assumption~\ref{ass:prog_sharp}. Therefore, $\norm{\nabla L_t} \leq O\left(\sfrac{\alpha_t}{\norm{\nabla S_t^\perp}}\right) \leq O(\sfrac{\alpha_t}{\rho_3}) \leq O(\sfrac{\epsilon\delta}{\eta})$. This gives that $\langle\nabla L(\param_t),\v{v}_t\rangle \leq O(\sfrac{\epsilon\delta^2}{\eta})$. For the second-order term, we have that
\begin{align*}
    \v{v}_t^\top\nabla^2L(\param_t^\dagger)\v{v}_t & = \hat{\v{v}}_t^\top\nabla^2L(\param_t^\dagger)\hat{\v{v}}_t + O\left(\norm{\nabla^2L(\param_t^\dagger)}\norm{\v{v}_t}\norm{\v{v}_t - \hat{\v{v}}_t}\right)\\
    & = \hat{\v{v}}_t^\top\nabla^2L(\param_t^\dagger)\hat{\v{v}}_t + O\left(\frac{\epsilon\delta^2}{\eta}\right)\\
    & = \left(\hat{\v{v}}_t^\perp + \langle\hat{\v{v}}_t,  \v{u}_t\rangle \v{u}_t\right)^\top\nabla^2L(\param_t^\dagger)\left(\hat{\v{v}}_t^\perp + \langle\hat{\v{v}}_t,  \v{u}_t\rangle \v{u}_t\right) + O\left(\frac{\epsilon\delta^2}{\eta}\right)\\
    & = \frac{2}{\eta}\hat{x}_t^2 + \hat{\v{v}}_t^\top\nabla^2L(\param_t^\dagger)\hat{\v{v}}_t + O\left(\frac{\epsilon\delta^2}{\eta}\right)\\
    & = \frac{2}{\eta}\hat{x}_t^2 + O\left(\frac{\epsilon\delta^2}{\eta}\right)
\end{align*}
where the last step follows from Assumption~\ref{ass:non_worst}. Putting things together, and noticing that $O(\rho_3\norm{\v{v}_t}^2) = O(\rho_3\delta^2)= O(\sfrac{\epsilon\delta}{\eta})$, we have that
\[
    L(\param_t) = L(\param_t^\dagger) + \frac{\hat{x}_t^2}{\eta} + O\left(\frac{\epsilon\delta^2}{\eta}\right)
\]
\end{proof}

\section{Proof of the Sharpness Gap Theorem}\label{app:gap_proof}

\begin{proof}[Proof of Theorem~\ref{thm:sharp_gap}]
Let the condition in Theorem~\ref{thm:coupling} hold. Then we have
\begin{equation}\label{eq:good_event}
  \norm{\v{v}_t - \hat{\v{v}}_t} \le O(\delta),
  \qquad
  \norm{\v{v}_t} \le O(\delta),
  \qquad
  \norm{\hat{\v{v}}_t} \le O(\delta)
\end{equation}
hold simultaneously for all $t \le \mathscr{T}$. This gives that $|\hat{x}_t|\leq \norm{\hat{\v{v}}_t} \leq O(\delta)$ and $|\hat{y}_t|\leq \norm{\nabla S_t}\norm{\hat{\v{v}}_t} \leq O(\rho_3\delta)$.
This also implies that $\eta|\hat{y}_t| \leq O(\eta\rho_3\delta) = O(\epsilon)$ and $\eta|\kappa| \leq \eta\epsilon\norm{\nabla S_t} \leq O(\eta\epsilon\rho_3)$. With the assumption that all landscape parameters are frozen (Assumption~\ref{ass:frozen}), we shall start the proof. The $\hat{y}$ formula from Lemma~\ref{lem:stoch_pred_xy} is:
\begin{gather}\label{eq:yhat_formula_proof}
  \hat{y}_{t+1}
  = \mathcal{I}_{t+1}
    + \eta\sum_{s=0}^{t}\beta_{s\rightarrow  t}
      \frac{\delta_s^{2}-\hat{x}_s^{2}}{2}
    - \eta\sum_{s=0}^{t}\langle \v{\gamma}_{s\rightarrow  t},\v{\xi}_s\rangle\\ \mathcal{I}_{t+1}
= (\nabla S_{t+1}^{\perp})^{\top}
\bigl[\prod_{k=t}^{0}\mat{A}_k\bigr]
\hat{\v{v}}_0^{\perp};\quad \beta_{s\rightarrow t} = (\nabla S_{t+1}^\perp)^\top\left[\prod_{k=t}^{s+1}\mat{A}_k\right]\nabla S_s^\perp
\end{gather}
where $\mat{A}_k = ( \mat{I} - \eta \mat{H}_k)P_{\v{u}_k}^\perp$. By Assumption~\ref{ass:frozen}, we have that $\mat{H}_k = \mat{H}$ and $\v{u}_k = \v{u}$. This gives that $\mat{A}_k = \mat{I} - \eta\mat{H}\mid_{\v{u}^\perp}$. Since $\nabla S_t^\perp = \nabla S^\perp \in \text{ker}(\mat{H}\mid_{\v{u}^\perp})$, we have that $\mat{A}_k\nabla S_t^\perp = \nabla S_t^\perp = \nabla S^\perp$. Therefore, $\mathcal{I}_{t+1} = \langle \nabla S^\perp, \hat{\v{v}}_0^\perp\rangle $. Similarly, we also have that
$\beta_{s\rightarrow  t}
= (\nabla S^{\perp})^{\top}
\mat{A}^{t-s}\nabla S^{\perp}
= (\nabla S^{\perp})^{\top}
\nabla S^{\perp}
= \norm{\nabla S^{\perp}}^{2} = \beta$
for all $s,t$. Therefore, under Assumption~\ref{ass:frozen}, the dynamic of $\hat{y}_t$ and $\hat{x}_t$ boils down to
\begin{align*}
    \hat{y}_{t+1} & = \langle \nabla S^\perp, \hat{\v{v}}_0^\perp\rangle + \frac{\eta\beta}{2}\sum_{s=0}^t(\delta - \hat{x}_s^2) - \eta\sum_{s=0}^t\v{\gamma}_{s\rightarrow t}\cdot\v{\xi}_s\\
    \hat{x}_{t+1} & = - (1 + \eta\hat{y}_t)\hat{x}_t - \frac{\eta}{2}\kappa\hat{x}_t^2 - \eta\zeta_t
\end{align*}
Taking expectation gives
\begin{align*}
    \E{\hat{y}_{t+1}} & = \E{\hat{y}_{t}} + \frac{\eta\beta}{2}(\delta - \E{\hat{x}_t^2})\\
    \E{\hat{x}_{t+1}} & = \E{\left((1 + \eta\hat{y}_t)\hat{x}_t - \frac{\eta}{2}\kappa\hat{x}_t^2 - \eta\zeta_t\right)^2}
\end{align*}
since the gradient noise has zero-mean. Setting $\E{\hat{y}_{t+1}}=\E{\hat{y}_t}$ gives
\begin{equation}\label{eq:x2_stationary_proof}
  \E{\hat{x}_t^{2}} = \delta^{2}
  = \frac{2\alpha}{\beta}.
\end{equation}
\medskip\noindent
From the $\hat{x}$-update~\eqref{eq:xhat_update} under frozen
coefficients:
\[
  \hat{x}_{t+1}
  = -(1+\eta\hat{y}_t)\hat{x}_t
    - \tfrac{\eta}{2}\kappa\hat{x}_t^{2}
    - \eta\zeta_t.
\]
Define
$D_t := (1+\eta\hat{y}_t)\hat{x}_t
+ \frac{\eta}{2}\kappa\hat{x}_t^{2}$,
so that $\hat{x}_{t+1} = -D_t - \eta\zeta_t$.
Squaring:
\begin{equation}\label{eq:x2_squared_proof}
  \hat{x}_{t+1}^{2}
  = D_t^{2}
    + 2\eta\zeta_t D_t
    + \eta^{2}\zeta_t^{2}.
\end{equation}
$D_t$ is $\calF_t$-measurable (it depends on
$\hat{x}_t,\hat{y}_t$, which are functions of
$B_0,\ldots,B_{t-1}$), while $\zeta_t$ depends on $B_t$,
which is independent of $\calF_t$.
Taking conditional expectations:
\begin{align}\label{eq:x2_cond_expanded}
  \E{\hat{x}_{t+1}^{2}\mid\calF_t}
  &= \E{D_t^{2}\mid\calF_t}
    + \E{2\eta\zeta_t D_t\mid\calF_t}
    + \E{\eta^{2}\zeta_t^{2}\mid\calF_t} \nonumber\\
  &= D_t^{2}
    + 2\eta D_t\underbrace{\E{\zeta_t\mid\calF_t}}_{=\,0}
    + \eta^{2}\E{\zeta_t^{2}\mid\calF_t} \nonumber\\
  &\stackrel{\ref{ass:frozen}\text{(iii)}}{=}
    D_t^{2}
    + \eta^{2}\sigma_{ \v{u}}^{2}
    + O(\epsilon^{2}\delta^{2}).
\end{align}
Now taking full expectations and applying the property
$\E{\E{\hat{x}_{t+1}^{2}\mid\calF_t}} = \E{\hat{x}_{t+1}^{2}}$:
\begin{align}\label{eq:expect}
  \E{\hat{x}_{t+1}^{2}} &= \E{D_t^{2}
    + \eta^{2}\sigma_{ \v{u}}^{2}
    + O(\epsilon^{2}\delta^{2})} \nonumber\\
  &= \E{D_t^{2}} + \eta^{2}\sigma_{ \v{u}}^{2} + O(\epsilon^{2}\delta^{2}).
\end{align}
At stationarity,
$\E{\hat{x}_{t+1}^{2}}=\E{\hat{x}_t^{2}}=\delta^{2}$,
so:
\begin{equation}\label{eq:x2_stat_Dt}
  \delta^{2}
  = \E{D_t^{2}}
    + \eta^{2}\sigma_{ \v{u}}^{2}
    + O(\epsilon^{2}\delta^{2}).
\end{equation}
We now show that
$\E{D_t^{2}}
= \E{(1+\eta\hat{y}_t)^{2}\hat{x}_t^{2}}
+ O(\epsilon^{2}\delta^{2})$.
Expanding $D_t^{2}$:
\begin{align}\label{eq:Dt2_expansion_proof}
  D_t^{2}
  = \left((1+\eta\hat{y}_t)\hat{x}_t
    + \tfrac{\eta}{2}\kappa\hat{x}_t^{2}\right)^{2}
    \nonumber = (1+\eta\hat{y}_t)^{2}\hat{x}_t^{2}
    + \eta\kappa(1+\eta\hat{y}_t)\hat{x}_t^{3}
    + \frac{\eta^{2}\kappa^{2}}{4}\hat{x}_t^{4}
\end{align}
Due to the upper bound on $\hat{y}_t,\hat{x}_t$ and $\kappa$, we have that
\begin{gather}
    \frac{\eta^2\kappa^2}{4}\hat{x}_t^4 \leq O(\eta^2\epsilon^2\rho_3^2\delta^4) \leq O(\epsilon^4\delta^2) \leq O(\epsilon^2\delta^2)\\
    \left|\eta\kappa(1+\eta\hat{y}_t)\hat{x}_t^{3}\right| \leq O(\eta\cdot\epsilon\rho_3\cdot(1+\epsilon)\cdot \delta^3) \leq O(\eta\epsilon\rho_3\delta^3) \leq O(\epsilon^2\delta^2)
\end{gather}
This shows that $\E{D_t^{2}}
= \E{(1+\eta\hat{y}_t)^{2}\hat{x}_t^{2}}
+ O(\epsilon^{2}\delta^{2})$.
Substituting into~\eqref{eq:x2_stat_Dt}:
\begin{equation}\label{eq:x2_stat_simplified_proof}
  \delta^{2}
  = \E{(1+\eta\hat{y}_t)^{2}\hat{x}_t^{2}}
    + \eta^{2}\sigma_{ \v{u}}^{2}
    + O(\epsilon^{2}\delta^{2}).
\end{equation}
\medskip\noindent
By Assumption~\ref{ass:decorr}(a) and, at equilibrium,
$\E{\hat{x}_t^{2}}=\delta^{2}$, we have that
\begin{equation}\label{eq:mf_applied_proof}
  \E{(1+\eta\hat{y}_t)^{2}\hat{x}_t^{2}}
  = (1+\eta\E{\hat{y}_t})^{2}\delta^{2}
    + O(\epsilon^{2}\delta^{2}).
\end{equation}
Substituting~\eqref{eq:mf_applied_proof} into~\eqref{eq:x2_stat_simplified_proof}:
\begin{align}\label{eq:quadratic_ybar_proof}
  \delta^{2}
  & = (1+\eta\E{\hat{y}_t})^{2}\delta^{2}
    + \eta^{2}\sigma_{ \v{u}}^{2}
    + O(\epsilon^{2}\delta^{2})\\
    & = \delta^2 + 2\eta\E{\hat{y}_t}\delta^2 + \eta^2\E{\hat{y}_t}^2\delta^2 + \eta^2\sigma_{ \v{u}}^{2}
    + O(\epsilon^{2}\delta^{2})\\
    & = (1 + 2\eta\E{\hat{y}_t})\delta^2 +\eta^2\sigma_{ \v{u}}^{2}
    + O(\epsilon^{2}\delta^{2})
\end{align}
where we used the fact that $\left|\eta\E{\hat{y}_t}\right| \leq O(\epsilon)$. Therefore, solving for $\E{\hat{y}_t}$ gives that, at equilibrium,
\[
    \E{\hat{y}_t} = -\frac{\eta\sigma_{\v{u}}^2}{2\delta^2} + O\left(\frac{\epsilon^2}{\eta}\right) = -\frac{\eta\beta\sigma_{\v{u}}^2}{4\alpha} + O\left(\frac{\epsilon^2}{\eta}\right)
\]
In the meantime, recall that $|\kappa_t| \leq O(\epsilon\rho_3)$ and $|\hat{x}_t| \leq \delta$. Therefore $|\kappa_t\hat{x}_t| \leq O(\epsilon\rho_3\delta) \leq O\left(\frac{\epsilon^2}{\eta}\right)$. By Theorem~\ref{thm:coupling}, we have that at equilibrium
\[
    \E{S(\param_t)} = \frac{2}{\eta} + \E{\hat{y}_t} + O\left(\frac{\epsilon^2}{\eta}\right) = \frac{2}{\eta } - \frac{\eta\beta\sigma_{\v{u}}^2}{4\alpha} + O\left(\frac{\epsilon^2}{\eta}\right)
\]
\end{proof}

\section{Further Experiment Details and Results}\label{app:further_exp}
\subsection{Setup}\label{sec:exp-setup}

\medskip \noindent \textbf{Architectures and losses.}
Our primary experimental setup follows \cite{cohen2021eos}: a
fully-connected network with two hidden layers of 200 units each
and $\tanh$ activation (FC-Tanh), trained with mean squared error
(MSE) loss on CIFAR-10. This setup produces clean edge-of-stability
behavior with a well-separated top Hessian eigenvalue. To verify
generality across architectures, we also test a convolutional
network (CNN: two convolutional layers with $\tanh$ activation
followed by two fully-connected layers) trained with MSE loss,
as well as a fully-connected network with ReLU activation (FC-ReLU)
with MSE loss. We additionally test the CNN with cross-entropy loss
to examine the role of the loss function
(see Section~\ref{sec:exp-discussion}).

\medskip \noindent \textbf{Training protocol.}
We train on a random subset of 5{,}000 CIFAR-10 training images
using vanilla SGD (no momentum, no weight decay) with a fixed
learning rate $\eta$. The default learning rate is $\eta = 0.01$,
giving a stability threshold of $2/\eta = 200$. For the learning
rate scaling experiment, we vary $\eta \in \{0.005, 0.008, 0.01,
0.015, 0.02\}$. Each experiment is repeated across 3--5 random
seeds for statistical significance.

\medskip \noindent \textbf{Metrics.}
At regular intervals during training, we compute:
\begin{itemize}[nosep, leftmargin=*]
    \item The top Hessian eigenvalue $\lambda_{\max}(\nabla^2 L)$
    (sharpness) via Lanczos with Hessian-vector products;
    \item The projected gradient noise variance
    $\sigma_u^2 = \mathrm{Var}_B[\langle g_B, u\rangle]$
    estimated from 50 random mini-batches;
    \item The batch sharpness
    $\mathrm{BS}(\theta) = \mathbb{E}_B[g_B^\top H_B g_B / \|g_B\|^2]$
    estimated from 30 random mini-batches;
    \item Quantities $\alpha = -\langle \nabla L,
    \nabla S \rangle$ and
    $\beta = \|\nabla S^\perp\|^2$ via the sharpness gradient
    $\nabla S = \nabla^3 L(u,u)$ measured along the SGD trajectory as approximation of the corresponding quantities along the constrained trajectory.
\end{itemize}
We define the \emph{equilibrium sharpness} $S_{\mathrm{eq}}$ as
the average of $\lambda_{\max}$ over the last 20 recorded
measurements, after the sharpness trajectory has stabilized near
the edge of stability.

\medskip \noindent \textbf{Hardware.}
All experiments are run on NVIDIA RTX 6000 Ada GPUs (48\,GB).
The total computational budget is approximately 20 GPU-hours
across 172 individual training runs.

\subsection{Equilibrium Sharpness Across Conditions}\label{sec:exp-summary}

Figure~\ref{fig:summary}A provides a comprehensive view of the
equilibrium sharpness across all batch sizes tested for the FC-Tanh
architecture. The monotonic increase from $S_{\mathrm{eq}} \approx 190.8$
at $b = 50$ to $S_{\mathrm{eq}} \approx 207.9$ at $b = 2500$ (matching
the GD baseline within noise) demonstrates that the stochastic
self-stabilization effect is robust and spans a wide range of
batch sizes.

Figure~\ref{fig:summary}B shows the equilibrium sharpness for both
GD and SGD ($b = 200$) across five learning rates. Both GD and SGD
sharpness scale with $2/\eta$ as expected, with a consistent gap
between GD and SGD at each learning rate. Interestingly, the
absolute gap $\Delta S$ is larger at smaller $\eta$ (4.27 at
$\eta = 0.005$ versus 2.29 at $\eta = 0.02$). This occurs because
the landscape quantities $\alpha$ and $\beta$ depend on $\eta$
through the equilibrium point: at smaller $\eta$, the system
equilibrates at higher sharpness where the landscape geometry
may amplify the noise effect.

% \begin{figure}[t]
%     \centering
%     \includegraphics[width=\textwidth]{figures/fig3_equilibrium_summary.pdf}
%     \caption{\textbf{(A)}~Equilibrium sharpness versus batch size
%     for FC-Tanh + MSE ($\eta = 0.01$). The dashed line marks the
%     GD baseline $S_{\mathrm{GD}} = 207.8$.
%     \textbf{(B)}~Equilibrium sharpness for GD (gray) and SGD with
%     $b = 200$ (blue) across learning rates. SGD consistently
%     equilibrates below GD.}
%     \label{fig:summary}
% \end{figure}

\subsection{Batch Sharpness Saturation}\label{sec:exp-bs}

In our theory, we predict that the
\emph{batch sharpness}---the curvature in the direction of the
mini-batch gradient---approaches $2/\eta$ even as the full-batch
sharpness $\lambda_{\max}(\nabla^2 L)$ remains suppressed below
$S_{\mathrm{GD}}$. This resolves an apparent paradox: how can SGD
remain at the ``edge of stability'' if its sharpness is below $2/\eta$?
The answer is that the \emph{relevant} curvature for mini-batch
dynamics is the batch sharpness, not the full-batch sharpness.

Figure~\ref{fig:batch-sharpness} confirms this prediction across
three batch sizes ($b = 100, 500, 2000$). In all cases, the
full-batch sharpness (blue curve) equilibrates below $S_{\mathrm{GD}}$,
while the batch sharpness measurements (red, shown as a running
median with interquartile range) cluster near $2/\eta = 200$.
The effect is most pronounced at $b = 100$, where the full-batch
sharpness is $\approx 200$ while the GD baseline is $\approx 208$,
yet the batch sharpness reaches $\approx 183$.
For $b = 2000$, both sharpness measures are close to the GD value,
consistent with the vanishing gap as $b \rightarrow  n$. The violin plots
at the right edge of each panel show the distribution of batch
sharpness measurements during the equilibrium phase, confirming
concentration near $2/\eta$.

\begin{figure}[t]
    \centering
    \includegraphics[width=\textwidth]{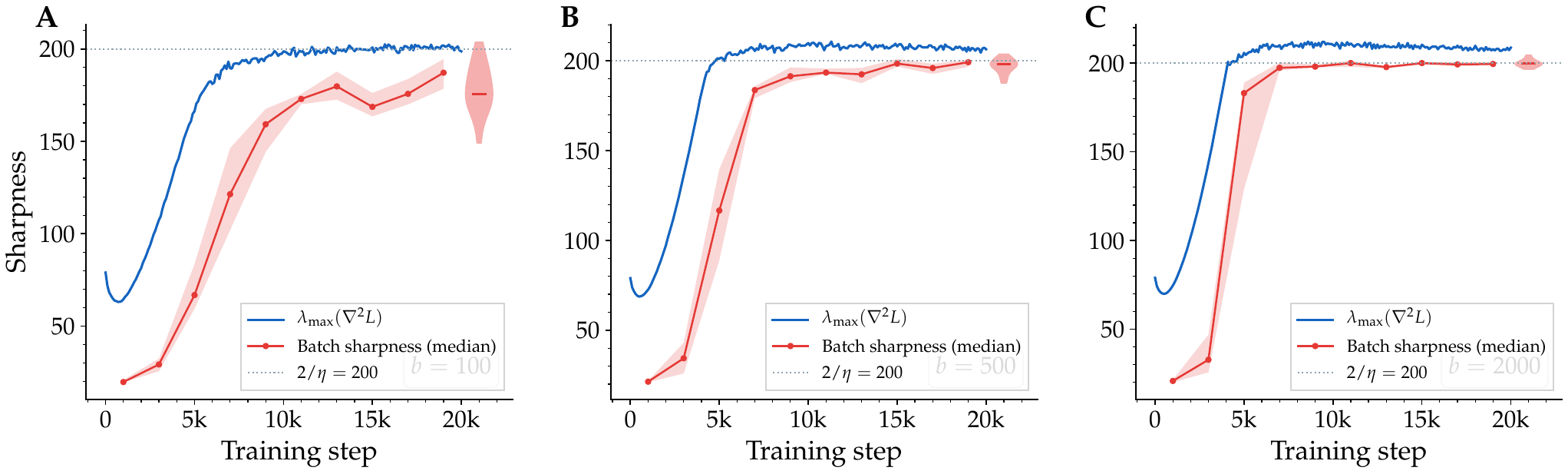}
    \caption{Batch sharpness versus full-batch sharpness
    during training for batch sizes $b \in \{100, 500, 2000\}$.
    Blue: $\lambda_{\max}(\nabla^2 L)$ (full-batch sharpness).
    Red: batch sharpness running median with interquartile range
    shading. The dotted line marks $2/\eta = 200$.
    Violin plots show the equilibrium distribution of batch
    sharpness, confirming our theory.}
    \label{fig:batch-sharpness}
\end{figure}

\subsection{Landscape Quantities}\label{sec:exp-landscape}

Figure~\ref{fig:landscape} shows the evolution of the key landscape
quantities during training for a representative run ($b = 200$,
$\eta = 0.01$). The sharpness (panel~A) exhibits the characteristic
progressive sharpening followed by EoS oscillation. The projected
noise variance $\sigma_u^2$ (panel~B) fluctuates but remains
consistently positive throughout training, confirming that gradient
noise along the top eigenvector is non-negligible. The
self-stabilization strength $\beta = \|\nabla S^\perp\|^2$ (panel~C)
increases during progressive sharpening and stabilizes at a large
positive value ($\beta \approx 7 \times 10^5$), confirming that
the restoring force is active at the equilibrium.

The progressive sharpening coefficient $\alpha = -\langle \nabla L,
\nabla S \rangle$ (panel~D) presents an important measurement
challenge. At the oscillating SGD iterate, $\alpha$ is frequently
negative because the iterate overshoots the stability threshold
$2/\eta$ during EoS oscillation. The theory predicts $\alpha > 0$
at the PGD reference trajectory (which lies on the manifold
$\{S(\theta) \leq 2/\eta\}$), not at the instantaneous SGD iterate.
This discrepancy between the measurement point and the theoretical
reference point means that the \emph{quantitative} gap prediction
$\Delta S = \eta \beta \sigma_u^2 / (4\alpha)$ requires care in
practice: one must either measure at a time-averaged iterate or
use the GD equilibrium landscape quantities. The \emph{qualitative}
predictions---the direction, monotonicity, and power-law scaling
of the gap---are robustly confirmed without needing to resolve
this measurement issue.

\begin{figure}[t]
    \centering
    \includegraphics[width=\textwidth]{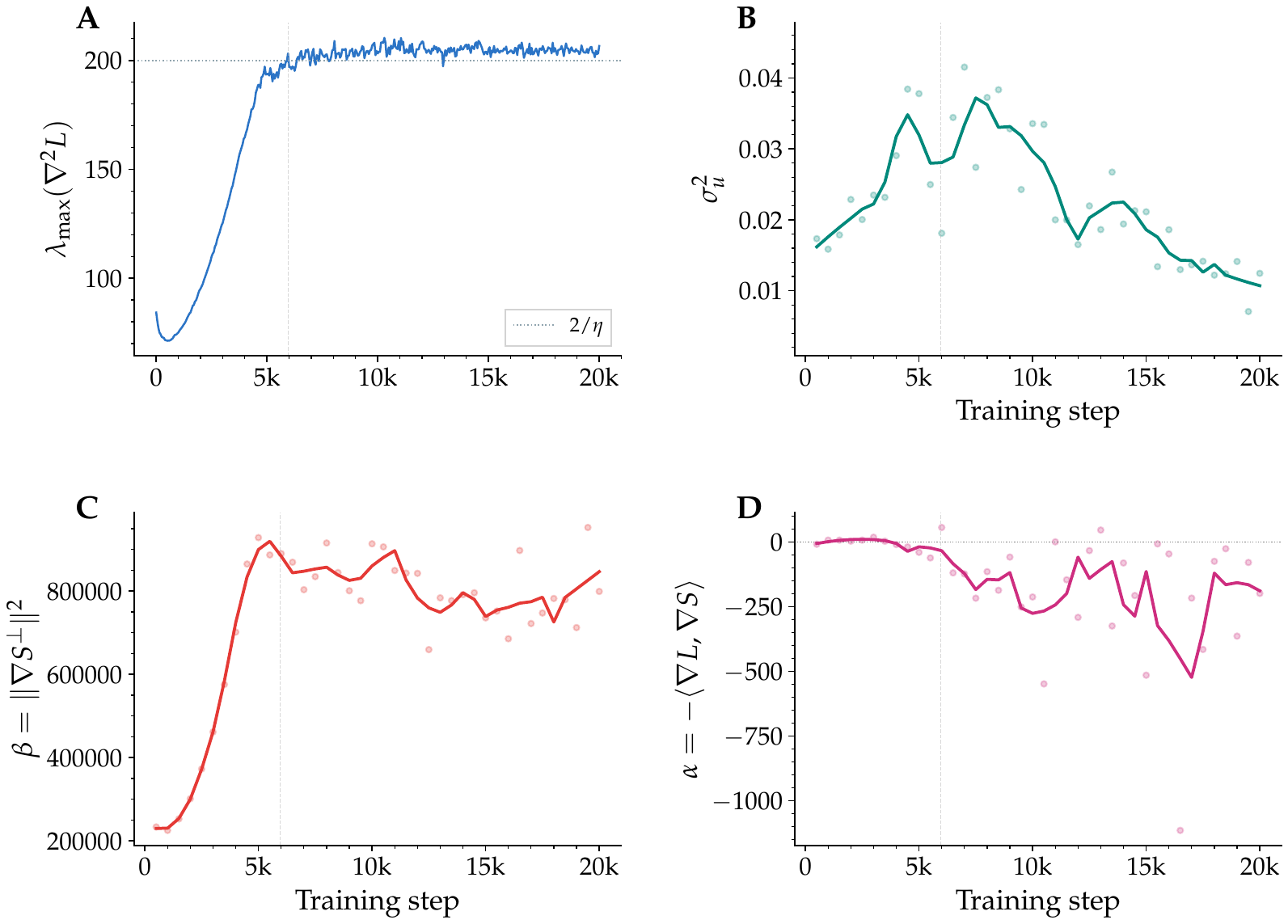}
    \caption{Landscape quantities during training ($b = 200$,
    $\eta = 0.01$).
    \textbf{(A)}~Sharpness with the $2/\eta$ threshold.
    \textbf{(B)}~Projected noise variance $\sigma_u^2$.
    \textbf{(C)}~Self-stabilization strength
    $\beta = \|\nabla S^\perp\|^2$.
    \textbf{(D)}~Progressive sharpening rate $\alpha$; negative
    values reflect measurement at the oscillating iterate rather
    than the PGD reference trajectory (see text).
    Raw measurements are shown as semi-transparent dots;
    solid lines show Savitzky--Golay smoothed trends.}
    \label{fig:landscape}
\end{figure}

\begin{figure}[t]
    \centering
    \includegraphics[width=\textwidth]{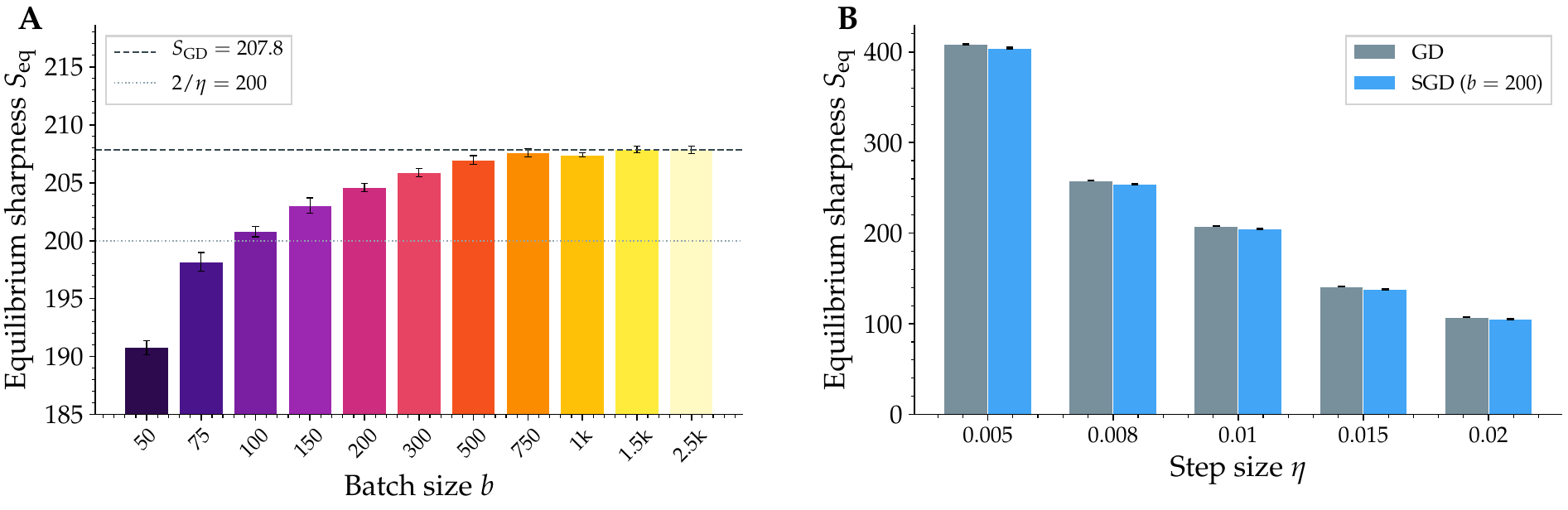}
    \caption{\textbf{(A)}~Equilibrium sharpness versus batch size
    for FC-Tanh + MSE ($\eta = 0.01$). The dashed line marks the
    GD baseline $S_{\mathrm{GD}} = 207.8$.
    \textbf{(B)}~Equilibrium sharpness for GD (gray) and SGD with
    $b = 200$ (blue) across learning rates. SGD consistently
    equilibrates below GD.}
    \label{fig:summary}
\end{figure}

\subsection{Verification of the Mean-Field Approximation}\label{app:mean-field-exp}

\begin{wrapfigure}{r}{0.5\textwidth}
    \centering
    \includegraphics[width=\linewidth]{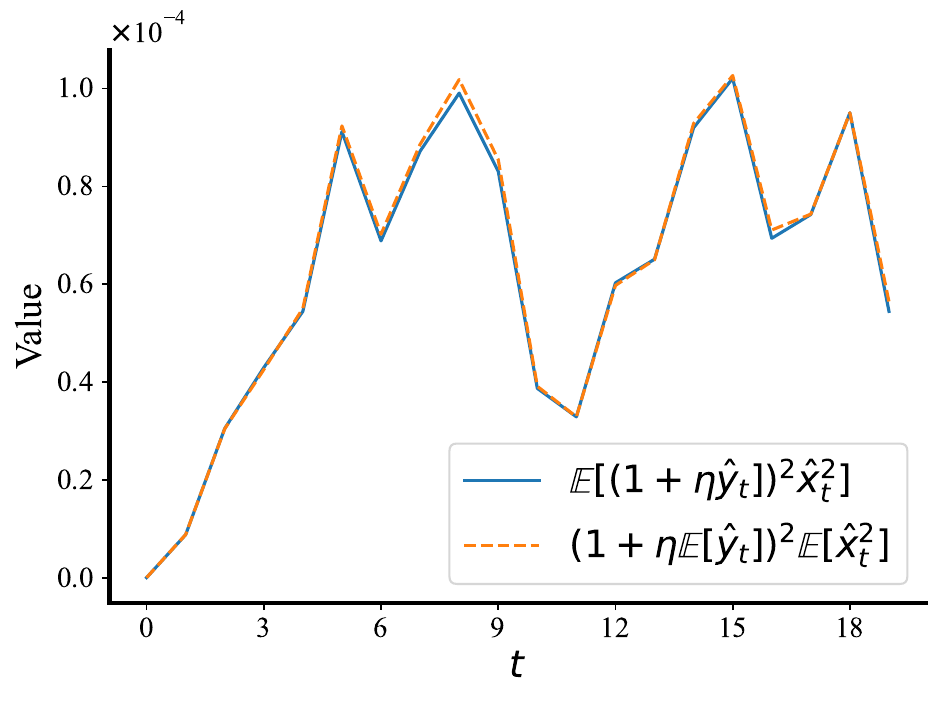}
    \caption{Dynamic of the two quantities involved in Assumption~\ref{ass:decorr}}
    \label{fig:mean_field}\vspace{-0.5cm}
\end{wrapfigure}

We verify that the mean-field approximation error assumed in Assumption~\ref{ass:decorr} is small in practice. To do so, we took a parameter checkpoint in the GD training at the time where the sharpness just entered EoS and use it as $\paramdagger_0$. Starting from this checkpoint, we generate the constrained trajectory using the update rule in Lemma~\ref{lem:dagger_step} using a learning rate of $\eta = 0.1$. Based on the constrained trajectory, the sequence $\hat{x}_t$ and $\hat{y}_t$ are generated according to the Definition~\ref{def:stoch_pred}. Here $\v{\xi}_t$ are generated by running a separate mini-batch SGD with batch size 128 starting from $\paramdagger_0$. We repeat this procedure 10 times to obtain a set of sample trajectories of $\hat{x}_t$ and $\hat{y}_t$. Based on the sample trajectories, the two quantities in Assumption~\ref{ass:decorr} are computed at each time step $t$, and plotted in Figure~\ref{fig:mean_field}. One could observe that the two values closely track each other, justifying that the mean-field approximation error is small.

\subsection{Discussion of Experimental Results}\label{sec:exp-discussion}

Our experiments provide strong support for the stochastic
self-stabilization mechanism across eight complementary tests:

\begin{itemize}[nosep,leftmargin=*]
    \item The batch-size scaling $\Delta S \propto b^{-1.27}$
    closely matches the theoretical $\Delta S \propto 1/b$,
    with high statistical significance ($R^2 = 0.98$).
    \item The noise variance scaling $\sigma_u^2 \propto b^{-1.21}$
    confirms the predicted mechanism: the projected gradient noise
    drives the sharpness suppression ($R^2 = 0.979$).
    \item Batch sharpness saturates near $2/\eta$ while full-batch
    sharpness remains suppressed, resolving the apparent paradox
    of SGD stability below the classical threshold.
    \item The product $\beta\sigma_u^2$ tracks the measured gap
    across batch sizes, confirming the decomposition predicted by
    the equilibrium formula.
    \item The gap scales as $\Delta S \propto \eta^{-0.47}$ with
    step size, reflecting the $\eta$-dependence of the equilibrium
    landscape geometry.
    \item The effect generalizes across architectures:
    a CNN with MSE loss shows an even stronger gap
    ($\Delta S \approx 71$ at $b=50$), and an FC-ReLU network
    confirms the mechanism is not activation-specific
    ($\Delta S \approx 16$ at $b=50$).
\end{itemize}

The main experimental limitation is the difficulty of measuring
$\alpha$ at the correct reference point. We note that the
\emph{scaling} predictions (which test the functional form of the
dependence on~$b$ and~$\eta$) are more robust than the
\emph{quantitative} prediction (which requires accurate
measurements of~$\alpha$,~$\beta$, and~$\sigma_u^2$ at a specific
point on the loss landscape). Developing practical methods for
measuring landscape quantities at the PGD reference trajectory
is an interesting direction for future work.

\medskip \noindent \textbf{Role of the loss function.}
We also tested the same CNN architecture with cross-entropy (CE)
loss on the same 5{,}000 CIFAR-10 subset.
Figure~\ref{fig:cnn-ce}A shows the sharpness trajectories for
CNN+CE across four batch sizes and GD. Unlike the MSE experiments,
the CNN+CE setup did not exhibit edge-of-stability behavior:
sharpness peaked at $\approx 85$--$105$ (well below $2/\eta = 200$)
around step 4{,}000 and then declined to $\approx 14$--$18$ by
step 15{,}000. Crucially, neither the batch size nor the use of
full-batch GD meaningfully affected the equilibrium sharpness---all
trajectories collapsed to the same low value, in stark contrast to
the batch-size-dependent separation seen with MSE loss
(Figure~\ref{fig:cnn-ce}B, reproduced for comparison).

This collapse occurs because cross-entropy loss on a small
dataset drives the network to near-zero loss rapidly, at which
point the Hessian spectrum flattens as the model enters an
interpolation regime. The stochastic self-stabilization mechanism
requires sustained edge-of-stability dynamics---specifically,
a balance between progressive sharpening ($\alpha > 0$) and
noise-induced stabilization ($\beta \sigma_u^2 > 0$)---which
cannot operate when the sharpness is far below $2/\eta$.
This observation is consistent with prior work showing that
clean EoS behavior is more readily obtained with MSE loss
\citep{cohen2021eos}, and highlights that the \emph{loss
function}, not just the architecture, plays a critical role in
determining whether the edge-of-stability regime is reached.

\begin{figure}[t]
    \centering
    \includegraphics[width=\textwidth]{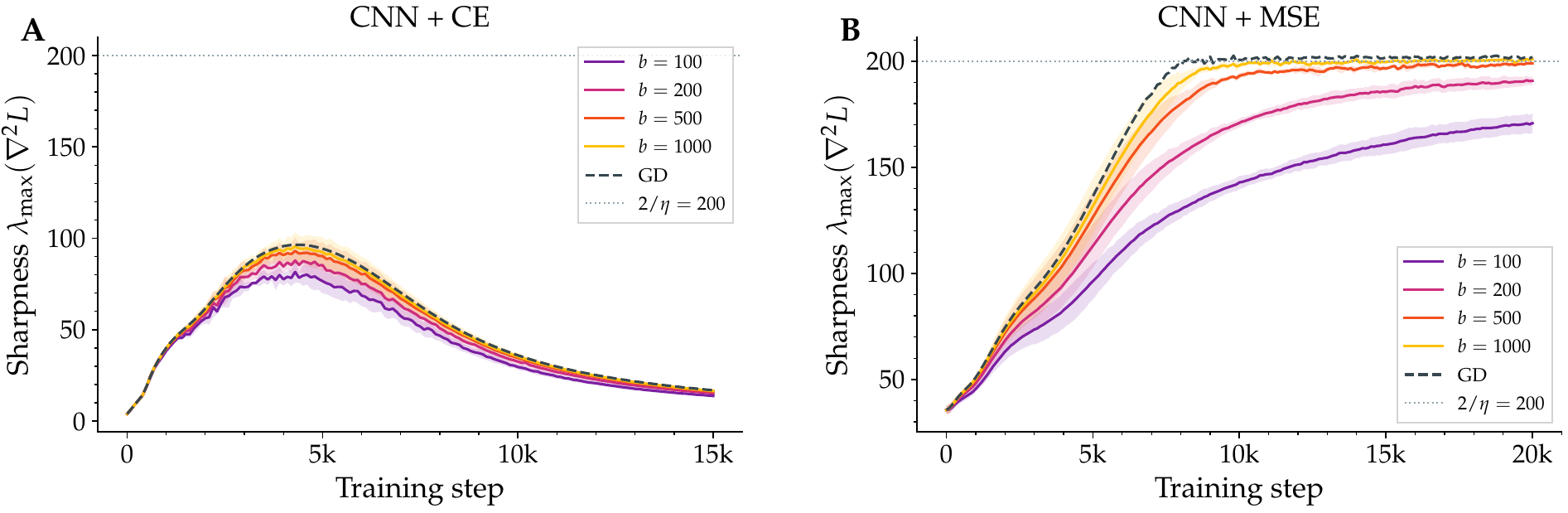}
    \caption{The role of the loss function.
    \textbf{(A)}~CNN + CE: sharpness peaks at $\approx 95$ and
    collapses to $\approx 15$, far below $2/\eta = 200$, with no
    batch-size dependence. Edge-of-stability dynamics never emerge.
    \textbf{(B)}~CNN + MSE (same architecture): clean EoS behavior
    with pronounced batch-size-dependent sharpness gap.
    Shaded regions indicate $\pm 1$ standard deviation across 3 seeds.}
    \label{fig:cnn-ce}
\end{figure}
\end{document}